\documentclass[11pt]{article}
\usepackage{enumerate}
\usepackage{pdfsync}
\usepackage[OT1]{fontenc}

\usepackage{url}           
\usepackage{booktabs}    
\usepackage{amsfonts}      
\usepackage{nicefrac}      
\usepackage{microtype}    
\usepackage{color}
\usepackage[colorlinks,
            linkcolor=red,
            anchorcolor=blue,
            citecolor=blue,
            pagebackref=true
           ]{hyperref}

\usepackage{fullpage}
\usepackage{setspace}
\usepackage{tabularx}
\usepackage{float}
\usepackage{wrapfig,lipsum}
\usepackage{enumitem}
\usepackage{natbib}
\usepackage{multirow} 
\usepackage[dvipsnames]{xcolor}

\makeatletter
\newcommand*{\rom}[1]{\expandafter\@slowromancap\romannumeral #1@}
\makeatother

\usepackage{mmll}

\def \dd{\text{d}}

\newcommand{\kl}{ {\rm{kl}} }
\newcommand{\nop}[1]{}

\def \algname{\text{ExpTS}}%
\def \hmu{\widehat{\mu}}

\renewcommand*{\backref}[1]{}
\renewcommand*{\backrefalt}[4]{%
\ifcase #1 %
    No citations.%
\or
    (p. #2.)%
\else
    (pp. #2.)%
\fi}%

\allowdisplaybreaks

\begin{document}
\title{\huge Finite-Time Regret of Thompson Sampling Algorithms for Exponential Family Multi-Armed Bandits}
\author
{
  Tianyuan Jin\thanks{National University of Singapore; e-mail: {\tt tianyuan@u.nus.edu}} 
  ,~
  Pan Xu\thanks{California Institute of Technology; e-mail: {\tt pan.xu@duke.edu}} 
  ,~
  Xiaokui Xiao\thanks{National University of Singapore; e-mail: {\tt xkxiao@nus.edu.sg}} 
  ,~
  Anima Anandkumar\thanks{California Institute of Technology; e-mail: {\tt anima@caltech.edu}} 
}
\date{}
\maketitle

\begin{abstract}
We study the regret of Thompson sampling (TS) algorithms for exponential family bandits, where the reward distribution is from a one-dimensional exponential family, which covers many common reward distributions including Bernoulli, Gaussian, Gamma, Exponential, etc. We propose a Thompson sampling algorithm, termed ExpTS, which uses a novel sampling distribution to avoid the under-estimation of the optimal arm. We provide a tight regret analysis for ExpTS, which simultaneously yields both the finite-time regret bound as well as the asymptotic regret bound. In particular, for a $K$-armed bandit with exponential family rewards, ExpTS over a horizon $T$ is sub-UCB (a strong criterion for the finite-time regret that is problem-dependent), minimax optimal up to a factor $\sqrt{\log K}$, and asymptotically optimal, for exponential family rewards. Moreover, we propose ExpTS$^+$, by adding a greedy exploitation step in addition to the sampling distribution used in ExpTS, to avoid the over-estimation of sub-optimal arms. ExpTS$^+$ is an anytime bandit algorithm and achieves the minimax optimality and asymptotic optimality simultaneously for exponential family reward distributions. Our proof techniques are general and conceptually simple and can be easily applied to analyze standard Thompson sampling with specific reward distributions.
\end{abstract}

\section{Introduction}

The Multi-Armed Bandit (MAB) problem is centered around a fundamental model for balancing the exploration versus exploitation trade-off in many online decision problems. In this problem,  the agent is given an environment with a set of $K$ arms $[K]=\{1,2,\cdots, K\}$.  At each time step $t$, the agent pulls an arm $A_t\in [K]$ based on observations of previous $t-1$ time steps, and then a reward $r_t$ is revealed at the end of the step. In real-world applications, reward distributions often  have different forms such as Bernoulli, Gaussian, etc. As suggested by \citet{auer2002finite,auer2002nonstochastic,agrawal2017near,lattimore2018refining,garivier2018kl}, a good bandit strategy should be general enough to cover a sufficiently rich family of reward distributions.  In this paper, we assume the reward $r_t$ is independently generated from some  canonical one-parameter exponential family of distributions with a mean value $\mu_{A_t}$. It is a rich family that covers many common distributions including Bernoulli, Gaussian, Gamma, Exponential, and others. 

The goal of a bandit strategy is usually to maximize the cumulative reward over $T$ time steps, which is equivalent to minimizing the regret, defined as the expected cumulative difference between playing the best arm and playing the arm according to the strategy:
$R_{\mu}(T)=T \cdot \max_{i\in [K]}\mu_{i}-\EE[\sum_{t=1}^{T} r_t ]$. 
We assume, without loss of generality, $\mu_1=\max_{i\in [K]} \mu_i$ is the best arm throughout this paper. For a fixed bandit instance (i.e., mean rewards $\mu_1,\cdots,\mu_K$ are fixed), \citet{lai1985asymptotically} shows that for  distributions that are continuously parameterized by their means,
\begin{align}
    \label{eq:def-asym}
    \lim_{T\rightarrow \infty} \frac{R_{\mu}(T)}{\log T}\geq \sum_{i>1}\frac{\mu_1-\mu_i}{\kl(\mu_i,\mu_1)},
\end{align}
where $\kl(\mu_i,\mu_1)$ is the Kullback-Leibler divergence between two distributions with mean $\mu_i$ and $\mu_1$.  A bandit strategy satisfying $ \lim_{T\rightarrow \infty}R_{\mu}(T)/\log T = \sum_{i>1}\frac{\mu_1-\mu_i}{\kl(\mu_i,\mu_1)}$ is said to be \textit{asymptotically optimal} or achieve the asymptotic optimality in regret. The asymptotic optimality is one of the most important statistical properties in regret minimization, which shows that an algorithm is consistently good when it is played for infinite steps and thus should be a basic theoretical requirement of any good bandit strategy \citep{auer2002finite}.  

In practice, we can only run the bandit algorithm for a finite number $T$ steps, which is the time horizon of interest in real-world applications. Therefore, the finite-time regret  is the ultimate property of a practical bandit strategy in regret minimization problems.  A strong notion of finite-time regret bounds is called \emph{the sub-UCB criteria} \citep{lattimore2018refining},  where the regret  satisfies
\begin{align}\label{def:problem-dependent-regret}
    R_{\mu}(T)=O\Bigg(\sum_{i\in [K]: \Delta_i>0}\bigg(\Delta_i+\frac{ \log T}{\Delta_i}\bigg) \Bigg),
\end{align}
where $\Delta_i=\mu_1-\mu_i$ is the sub-optimal gap between arm $1$ and arm $i$. Note that the regret bound in \eqref{def:problem-dependent-regret} is a problem-dependent bound since it depends on the bandit instance and the sub-optimal gaps. Sub-UCB is an important metric for finite-time regret bound and has been adopted by recent work of \citet{lattimore2018refining,bian2021maillard}. Another special type of finite-time bounds is called the worst-case regret, which is defined as the finite-time regret of an algorithm on any possible bandit instance within a bandit class. Specifically, for a finite time horizon $T$,  \citet{auer2002finite} proves that any strategy has at least worst-case regret $\Omega(\sqrt{KT})$ for a $K$-armed bandit. We say the strategy that achieves a worst-case regret $O(\sqrt{KT})$ is \textit{minimax optimal} or achieves the minimax optimality. Different from the asymptotic optimality, the minimax optimality characterizes the worst-case performance of the bandit strategy in finite steps. 

A vast body of literature in multi-armed bandits \citep{audibert2009minimax,agrawal2017near,kaufmann2016bayesian,menard2017minimax,garivier2018kl,lattimore2018refining} have been pursing the aforementioned theoretical properties of bandit algorithms: generality, asymptotic optimality,  problem-dependent finite-time regret, and minimax optimality.  However, most of them focus on one or two properties and sacrifice the others. Moreover, many of existing theoretical analyses of bandit strategies are for optimism-based algorithm. The theoretical analysis of Thompson sampling (TS) is much less understood until recently, which has been shown to exhibit superior practical performances compared to the state-of-the-art methods \citep{chapelle2011empirical,wang2018thompson}. Specifically, its finite-time regret, asymptotic optimality, and near minimax optimality have been studied by \citet{agrawal2012analysis,agrawal2013further,agrawal2017near} for Bernoulli rewards. \citet{jin2021mots} proved the minimax optimality of TS for sub-Gaussian rewards. For exponential family reward distributions, the asymptotic optimality is shown by \citet{korda2013thompson}, but no finite-time regret of TS is provided. See Table \ref{table:regret_comparison-ts} for a comprehensive comparison of these results.

In this paper, we study the regret of Thompson sampling for exponential family reward distributions and address all the theoretical properties of TS. We propose a  variant of TS algorithm with a general sampling distribution and a tight analysis for frequentist regret bounds. Our analysis simultaneously yields both the finite-time regret bound and  the asymptotic regret bound. Specifically, the \textbf{main contributions} of this paper are summarized as follows:
\begin{itemize}[leftmargin=*]
    \item We propose $\algname$, a general variant of Thompson Sampling, that uses a novel sampling distribution with a tight anti-concentration bound to avoid the under-estimation of the optimal arm and a tight concentration bound to avoid the over-estimation of sub-optimal arms. For exponential family of reward distributions, we prove that $\algname$ is the first Thompson sampling algorithm achieving the sub-UCB criteria, which is a strong notion of problem-dependent finite-time regret bounds. We further show that $\algname$ is also simultaneously minimax optimal up to a factor of  $\sqrt{\log K}$,  as well as asymptotically optimal,  where $K$ is the number of arms. 
    \item We also propose $\algname^{+}$, which explores between the sample generated in $\algname$ and the empirical mean reward for each arm, to get rid of the extra $\sqrt{\log K}$ factor in the worst-case regret. Thus $\algname^{+}$ is the first Thompson sampling algorithm that is simultaneously minimax and asymptotically optimal for exponential family of reward distributions.
    \item Our regret analysis of ExpTS can be easily extended to analyze Thompson sampling with common reward distributions in the exponential family. For Gaussian reward with known variance, we prove that TS with Gaussian posteriors (Gaussian-TS) is asymptotically optimal and minimax optimal up to a factor of $\sqrt{\log K}$. For Bernoulli reward, we prove that TS with Beta posteriors (Bernoulli-TS) is  asymptotically optimal and minimax optimal up to a factor of $\sqrt{\log K}$. Similar to the idea of \algname$^+$, we can add a greedy exploration step to the posterior distributions used in these variants of TS, and then the algorithms are simultaneously minimax and asymptotically optimal. 
\end{itemize}

Our techniques are novel and conceptually simple. First, we introduce a lower confidence bound in the regret decomposition to avoid the under-estimation of the optimal arm, which is important in obtaining the finite-time regret bound. Specifically, \citet{jin2021mots} (Lemma 5 in their paper) shows that for Gaussian reward distributions, Gaussian-TS has a regret bound at least in the order of $\Omega(\sqrt{KT\log T})$ if the standard regret decomposition in existing analysis of Thompson sampling \citep{agrawal2017near,lattimore2018bandit,jin2021mots} is adopted. With our new regret decomposition that is conditioned on the lower confidence bound introduced in this paper, we improve the worst-case regret of Gaussian-TS for Gaussian reward distributions to $O(\sqrt{KT\log K})$.  

Second, we do not require the closed form of the reward distribution, but only make use of the corresponding concentration bounds. This means our results can be readily extended to other reward distributions. For example, we can extend ExpTS$^+$ to sub-Gaussian reward distributions and the algorithm is simultaneously minimax and asymptotically optimal\footnote{Note that sub-Gaussian is a non-parametric family and thus the lower bound \eqref{eq:def-asym} by \citet{lai1985asymptotically} does not directly apply to a general sub-Gaussian distribution. Following similar work in the literature \citep{jin2021mots}, in this paper, when we say an algorithm achieves the asymptotic optimality for sub-Gaussian rewards, we mean its regret matches the asymptotic lower bound for Gaussian rewards, which is a stronger notion.}, which improve the results of MOTS proposed by \citet{jin2021mots} (see Table \ref{table:regret_comparison-ts}).

Third, the idea of ExpTS$^+$ is simple and can be used to remove the extra $\sqrt{\log K}$ factor in the worst-case regret.  We note that MOTS \citep{jin2021mots} can also achieve the minimax optimal via the clipped Gaussian. However, it is not clear how to generalize the clipping idea to the  exponential family of reward distribution. Moreover, it uses the MOSS \citep{audibert2009minimax} index for clipping, which needs to know the horizon $T$  in advance and thus cannot be extended to the anytime setting, while ExpTS$^+$ is an anytime bandit algorithm which does not need to know the horizon length in advance.

\begin{table*}[t]
{\footnotesize
    \centering
    \caption{Comparisons of different Thompson sampling algorithms on $K$-armed bandits over a horizon $T$. For any algorithm, \emph{Asym. Opt} is the indicator whether it is asymptotically optimal,  \emph{minimax ratio} is the scaling of its worst-case regret w.r.t. the minimax optimal regret $O(\sqrt{VKT})$, where $V$ is the variance of reward distributions, and sub-UCB is the indicator if it satisfies the sub-UCB criteria. 
    \label{table:regret_comparison-ts}}    
    \begin{tabular}{ccccccl}
    \toprule
     \multirow{2}{*}{Algorithm}  & \multirow{2}{*}{Reward Type}&  Asym. &   \multicolumn{2}{c}{Finite-Time Regret}  & \multirow{2}{*}{Anytime} & \multirow{2}{*}{Reference}\\
     \cline{4-5}
     &  & Opt &  Minimax Ratio  &Sub-UCB & &\\
    \midrule
    TS & Bernoulli & yes & $\sqrt{\log T}$ & yes & yes & \citep{agrawal2013further} \\
    TS & Bernoulli & -- & $\sqrt{\log K}$  & yes & yes & \citep{agrawal2017near} \\ 
    TS & Exponential Family & yes & -- & -- & yes &\citep{korda2013thompson} \\
    MOTS & sub-Gaussian  & no & 1 & no & no & \citep{jin2021mots} \\
    MOTS-$\mathcal{J}$ & Gaussian & yes & 1 & no & no & \citep{jin2021mots} \\
    $\algname$ & Exponential Family  & yes & $\sqrt{\log K}$ & yes & yes & This paper \\
    $\algname^{+}$ & Exponential Family & yes & 1 & no & yes & This paper \\
    \bottomrule
    \end{tabular}
}
\end{table*}

\noindent\textbf{Notations.} 
We let $T$ be the total number of time steps, $K$ be the total number of arms, and $[K]=\{1,2,\cdots,K\}$. For simplicity, we assume arm 1 is the optimal throughout this paper, i.e., $\mu_1=\max_{i\in [K]} \mu_i$. We denote $\log^+(x)=\max\{0, \log x\}$ and $\Delta_i:=\mu_1-\mu_i$, $i\in [K]\setminus \{1\}$ for 
the gap between arm 1 and arm $i$. We let $T_{i}(t):=\sum_{j=1}^{t}\ind\{A_t=i\}$ be the number of pulls of arm $i$ at the time step $t$, $\hmu_{i}(t):=1/T_{i}(t)\sum_{j=1}^{t} \big[r_j\cdot \ind\{ A_t=i \}\big]$ be the average reward of arm $i$ at the time step $t$, and $\hmu_{is}$  be the average reward of arm $i$ after its $s$-th pull.

\section{Related Work}
There are series of works pursuing the asymptotic regret bound and worst-case regret bound for MAB. For asymptotic optimality, UCB algorithms \citep{garivier2011kl,maillard2011finite,agrawal2017near,lattimore2018refining}, Thompson sampling \citep{kaufmann2012thompson,korda2013thompson,agrawal2017near,jin2021mots}, Bayes-UCB \citep{kaufmann2016bayesian}, and other methods \citep{jin2021double,bian2021maillard} are all shown to be asymptotically optimal. Among them, only a few \citep{garivier2011kl,agrawal2017near,korda2013thompson} can be extended to  exponential families of distributions. For the worst-case regret, MOSS \citep{audibert2009minimax} is the first algorithm proved to be minimax optimal.  Later, KL-UCB$^{++}$ \citep{agrawal2017near}, AdaUCB \citep{lattimore2018refining}, MOTS \citep{jin2021mots} also join the family. The anytime version of MOSS is studied by \citet{degenne2016anytime}. There are also some works that focus on the near optimal problem-dependent regret bound  \citep{lattimore2015optimally,lattimore2016regret}. As far as we know, no algorithm has been proved to achieve the sub-UCB criteria,  asymptotic optimality,  and minimax optimality simultaneously for exponential family reward distributions.

For Thompson sampling, \citet{russo2014learning} studied the Bayesian regret. They show that the Bayesian regret of Thompson sampling is never worse than the regret of UCB. \citet{bubeck2013prior} further showed the Bayesian regret of Thompson sampling is optimal using the regret analysis of MOSS. There are also a line of works focused on the frequentist regret of TS. \citet{agrawal2012analysis} proposed the first finite time regret analysis for TS. \citet{kaufmann2012thompson,agrawal2013further} proved that TS with beta posteriors is asymptotically optimal for Bernoulli reward distributions. \citet{korda2013thompson} extended the asymptotic optimality to the exponential family of reward distributions. Subsequently, for Bernoulli rewards, \citet{agrawal2017near} proved that TS with Beta prior is asymptotically optimal and has worst-case regret $O(\sqrt{KT\log T})$. Besides, they showed that TS with Gaussian posteriors can achieve a better worst-case regret bound  $O(\sqrt{KT\log K})$.  They also proved that for Bernoulli rewards, TS with Gaussian posteriors has a worst-case regret at least $\Omega(\sqrt{KT\log K})$.
Very recently, \citet{jin2021mots} proposed the MOTS algorithm that can achieve the minimax optimal regret $O(\sqrt{KT})$ for multi-armed bandits with sub-Gaussian rewards but at the cost of losing the asymptotic optimality by a multiplicative factor of $1/\rho$, where $0<\rho<1$ is an arbitrarily fixed constant. For bandits with Gaussian rewards, \citet{jin2021mots} proved that MOTS combined with a Rayleigh distribution can achieve the minimax optimality and the asymptotic optimality simultaneously. We refer readers to Tables \ref{table:regret_comparison-ts} and \ref{table:regret_comparison} for more details.

\section{Preliminary on Exponential Family Distributions}\label{sec:preliminary}
A one-dimensional canonical exponential family \citep{garivier2011kl,harremoes2016bounds,menard2017minimax} is a parametric set of probability distributions with respect to some reference measure, with the density function given by 
\begin{equation*}
    p_{\theta}(x)=\exp(x\theta-b(\theta)+c(x)),
\end{equation*}
where $\theta$ is the model parameter, and $c$ is a real function.   Denote the measure $p_{\theta}(x)$ as $\nu_{\theta}$. Then, the above definition can be rewritten as  
\begin{align*}
    \frac{\dd \nu_{\theta}}{\dd \rho}(x)=\exp(x\theta-b(\theta)),
\end{align*}
for some measure $\rho$ and $b(\theta)=\log ( \int e^{x\theta} \dd \rho(x))$.  We make the classic assumption used by  \citet{garivier2011kl,menard2017minimax} that $b(\theta)$ is twice differentiable with a continuous second derivative. Then, we can verify that exponential families have the following properties:
\begin{align*}
    b'(\theta)=\EE[\nu_{\theta}] \qquad \text{and} \qquad b''(\theta)=\Var[\nu_{\theta}]>0.
 \end{align*}
Let $\mu=\EE[\nu_{\theta}]$. The above equality  means that the mapping between the mean value $\mu$ of $\nu(\theta)$ and the parameter $\theta$ is one-to-one. 
Hence,  exponential family of distributions can also be parameterized by the mean value $\mu=b'(\theta)$. Note that $b''(\theta)>0$ for all $\theta$, which implies $b'(\cdot)$ is invertible and its inverse function $b'^{-1}$ satisfies $\theta=b'^{-1}(\mu)$. In this paper, we will use the notion of Kullback-Leibler (KL) divergence. The KL divergence between two exponential family distributions with parameter $\theta$ and $\theta'$ respectively is defined as follows:
\begin{align}
\label{eq:KL-def}
    {\rm KL}(\nu_{\theta},\nu_{\theta'})=b(\theta')-b(\theta)-b'(\theta)(\theta'-\theta).
\end{align}
Recall that the mapping $\theta \mapsto\mu$ is one-to-one. We can define an equivalent notion of the KL divergence between random variables $\nu_{\theta}$ and $\nu_{\theta'}$ as a function of the mean values $\mu$ and $\mu'$ respectively: 
\begin{align*}
    \kl(\mu,\mu')= {\rm KL}(\nu_{\theta},\nu_{\theta'}),
\end{align*}
where $\EE[\nu_{\theta}]=\mu$ and $\EE[\nu_{\theta'}]=\mu'$. Similarly, we define $V(\mu)=\Var(\nu_{b'^{-1}(\mu)})$ as the variance of an exponential family random variable $\nu_{\theta}$ with mean $\mu$. We assume the variances of  exponential family distributions used in this paper are bounded by a constant $V>0$: 
\begin{equation*}
     0<V(\mu)\leq V< +\infty.
\end{equation*} 
We have the following property of the KL divergence between exponential family distributions.
\begin{proposition}[\citet{harremoes2016bounds}]
\label{lem:kl2016}
Let $\mu$ and $\mu'$ be the mean values of two exponential family distributions. The Kullback-Leibler divergence between them can be calculated as follows:
\begin{align}
\label{eq:pro-eq1-1}
    \kl(\mu,\mu')=\int_{\mu}^{\mu'}\frac{x-\mu}{V(x)} \dd x.
\end{align}
\end{proposition}
It is worth noting that exponential families cover many of the most common distributions used in practice such as  Bernoulli, exponential, Gamma, and Gaussian distributions. In particular, for two Gaussian distributions with the same known variance $\sigma^2$  and different means $\mu$ and $\mu'$, we can choose $V(\cdot)=\sigma^2$, and recover the results in Proposition \ref{lem:kl2016} as $\kl(\mu,\mu')=(\mu-\mu')/(2\sigma^2)$. For two Bernoulli distributions with mean values $\mu$ and $\mu'$ respectively, we can choose $V(\cdot)=1/4$, and recover the result in Proposition \ref{lem:kl2016} as $\kl(\mu,\mu')=\mu\log(\mu/\mu')+(1-\mu)\log((1-\mu)/(1-\mu'))$.

Based on Proposition \ref{lem:kl2016}, we can also verify the following properties. 
\begin{proposition}[\citet{jin2021almost}]\label{prop:kl_ineq}
For all $\mu$ and $\mu'$, we have
\begin{equation}
\label{eq:pinsk}
    \kl(\mu,\mu')\geq (\mu-\mu')^2/(2V).
\end{equation}
In addition, for $\epsilon>0$ and $\mu\leq\mu'-\epsilon$,  we can obtain that
\begin{align}\label{eq:property-2}
    &\kl(\mu,\mu')\geq \kl(\mu,\mu'-\epsilon) \quad \text{and} \quad  \kl(\mu,\mu')\leq\kl(\mu-\epsilon,\mu').
\end{align}
\end{proposition}

\section{Thompson Sampling for Exponential Family Reward Distributions} 
\begin{algorithm}[t]
   \caption{Exponential Family Thompson Sampling ($\algname$)}
   \label{alg:rots}
\begin{algorithmic}[1]
   \STATE\textbf{Input:} Arm set $[K]$
   \STATE\textbf{Initialization:}
    Play each arm once and set $T_i(K)=1$; let $\hmu_{i}(K)$ be the observed  reward of playing arm $i$
   \FOR{$t=K+1,K+2,\cdots$}
   \STATE For all $i\in[K]$, sample $\theta_{i}(t)$  independently from  $\cP(\hmu_{i}(t), T_{i}(t))$
   \STATE Play arm $A_t= \arg \max_{i\in[K]} \theta_i(t)$ and observe the reward $r_t$
   \STATE For all $i\in[K]$, update the mean reward estimator and the number of pulls: $$\hat\mu_{i}(t)=\frac{T_{i}(t-1) \cdot \hmu_{i}(t-1)+r_t\ind\{i=A_t\}}{T_{i}(t-1)+\ind\{i=A_t\}},\quad  T_i(t)=T_i(t-1)+\ind\{i=A_t\}$$ 
    \ENDFOR
\end{algorithmic}
\end{algorithm}

We present a general variant of Thompson sampling for exponential family rewards in Algorithm \ref{alg:rots}, which is termed $\algname$. At round $t$, $\algname$ maintains an estimate of a sampling distribution for each arm, denoted as $\cP$. The algorithm generates a sample parameter $\theta_i(t)$ for each arm $i$ independently from their sampling distribution and chooses the arm that attains the largest sample parameter. For each arm $i\in[K]$, the sampling distribution $\cP$ is usually defined as a function of the total number of pulls $T_i(t)$ and the empirical average reward $\hmu_i(t)$. After pulling the chosen arm, the algorithm updates $T_i(t)$ and $\hmu_i(t)$ for each arm based on the reward $r_t$ it receives and proceeds to the next round. 

It is worth noting that we study the frequentist regret bound of Algorithm \ref{alg:rots} and thus $\algname$ is not restricted as a Bayesian method. As pointed out by \citet{abeille2017linear,jin2021mots,kim2021doubly,zhang2021feel}, the sampling distribution does not have to be a posterior distribution derived from a pre-defined prior distribution. Therefore, we call $\cP$ the sampling distribution instead of the posterior distribution as in Bayesian regret analysis of Thompson sampling \citep{russo2014learning,bubeck2013prior}. To obtain the finite-time regret bound of $\algname$ for exponential family rewards, we will discuss the choice of a general sampling distribution and a new proof technique. 

\subsection{Challenges in Regret Analysis for Exponential Family Bandits}
\label{sec:challenge}
Before we choose a specific sampling distribution $\cP$ for $\algname$, we first discuss the main challenges in the finite-time regret analysis of Thompson sampling, which is the main motivation for our design of $\cP$ in the next subsection.

\textbf{Under-Estimation of the Optimal Arm.} Denote  $\hmu_{is}$ as the average reward of arm $i$ after its $s$-th pull, $T_{i}(t)$ as the number of pulls of arm $i$ at time $t$,  and $\cP(\hmu_{is},s)$ as the sampling distribution of arm $i$ after its $s$-th pull. The regret of the algorithm contributed by pulling arm $i$ is $\Delta_i\EE[T_{i}(T)]$, where $T_{i}(T)$ is the total number of pulls of arm $i$. All existing analyses of finite-time regret bounds for TS \citep{agrawal2012analysis,agrawal2013further,agrawal2017near,jin2021mots} decompose this regret term as $\Delta_i\EE[T_{i}(T)]\leq D_i+h_i(\Delta_i, T, \theta_i(1),\ldots,\theta_i(T))$, where $h_i()$ is a quantity characterizing the over-estimation of arm $i$ which can be easily dealt with by some concentration properties of the sampling distribution (see Lemma \ref{lem:mini-overest-1} for more details). The term $D_i$ characterizes the under-estimation of the optimal arm $1$, which is usually bounded as follows in existing work:
\begin{align}
\label{eq:challenge0}
 D_i=\Delta_i\sum_{s=1}^{T} \EE_{\hmu_{1s}}\bigg[\frac{1}{G_{1s}(\epsilon)}-1\bigg],
\end{align}
where $G_{1s}(\epsilon)=1-F_{1s}(\mu_1-\epsilon)$, $F_{1s}$ is the CDF of the sampling distribution $\cP(\hmu_{1s},s)$, and $\epsilon=\Theta(\Delta_i)$. In other words, $G_{1s}(\epsilon)=\PP(\theta_{1}(t)>\mu_1-\epsilon)$ is the probability that the best arm will \emph{not be under-estimated} from the mean reward by a margin $\epsilon$. Furthermore, we can interpret the quantity in \eqref{eq:challenge0} as the result of a union bound indicating how many samples TS requires to ensure that at least one sample of the best arm $\{\theta_{1}(t)\}_{t=1}^{T}$ is larger than $\mu_1-\epsilon$. If $G_{1s}(\epsilon)$ is too small, arm $1$ could be significantly under-estimated, and thus $D_i$ will be unbounded. In fact, as shown in Lemma 5 by \citet{jin2021mots}, for MAB with Gaussian rewards, TS using Gaussian posteriors will unavoidably suffer from a regret $D_i=\Omega(\sqrt{KT\log T})$. 

To address the above issue, we introduce a lower confidence bound for measuring the under-estimation problem. We use a new decomposition of the regret that bounds $D_i$ with the following term
\begin{align}
\label{eq:challenge-lows}
    \Delta_i\sum_{s=1}^{T} \EE_{\hmu_{1s}} \left[\bigg( \frac{1}{G_{1s}(\epsilon)}-1\bigg) \cdot \ind\{\hmu_{1s}\geq Low_s \} \right],
\end{align}
where $Low_s$ is a lower confidence bound of $\hmu_{1s}$. Intuitively, due to the concentration of arm $1$'s rewards, the probability of $\hmu_{1s}\leq Low_s$ is very small. Thus, even when $G_{1s}(\epsilon)$ is small, the overall regret can be well controlled. In the regret analysis of TS, we can bound \eqref{eq:challenge-lows} from two facets: (1) the lower confidence bound can be proved using the concentration property of the reward distribution; and (2) the term $G_{1s}(\epsilon)=\PP(\theta_{1}(t)>\mu_1-\epsilon)$ can be upper bounded by the anti-concentration property for the sampling distribution $\cP$. To achieve an optimal regret, one needs to carefully balance the interplay between these two bounds. For a specific reward distribution (e.g., Gaussian, Bernoulli) as is studied by \citet{agrawal2017near,jin2021mots}, there are already tight anti-concentration inequalities for the reward distribution, and thus the  lower confidence bound is tight. Therefore, by choosing Gaussian or Bernoulli as the prior (which leads to a Gaussian or Beta sampling distribution $\cP$), we can use existing anti-concentration bounds for Gaussian \citep[Formula 7.1.13]{abramowitz1964handbook} or Beta \citep[Prop. A.4]{jevrabek2004dual} distributions to obtain a tight bound of $G_{1s}(\epsilon)$. 

In this paper, we study the general exponential family of reward distributions, which has no closed form. Thus we cannot obtain a tight concentration bound for $\hat{\mu}_{1s}$ as in special cases such as Gaussian or Bernoulli rewards. This increases the hardness of tightly bounding term \eqref{eq:challenge-lows} and it is imperative for us to design a sampling distribution $\cP$ with a tight anti-concentration bound that can carefully control $G_{1s}(\epsilon)$  without any knowledge of the closed form distribution of the average reward $\hat{\mu}_{1s}$. Due to the generality of exponential family distributions, it is challenging and nontrivial to find such a sampling distribution to obtain a tight finite-time regret bound.

\subsection{Sampling Distribution Design in Exponential Family Bandits}
\label{subsec:expon}
In this subsection, we show how to choose a sampling distribution $\cP$ that has a tight anti-concentration bound to overcome the under-estimation of the optimal arm and concentration bound to overcome the over-estimation of the suboptimal arms. 

For the simplicity of notation, we denote $\cP(\mu,n)$ as the sampling distribution, where $\mu$ and $n$ are some input parameters. In particular, for $\algname$, we will choose $\mu=\hmu_i(t)$ and $n=T_i(t)$ for arm $i\in[K]$ at round $t$. We define $\cP(\mu,n)$ as a distribution with PDF
\begin{align}\label{def:general_posterior}
f(x;\mu,n)&=1/2|(nb_n\cdot \kl(\mu,x))'| e^{-n b_n \cdot \kl(\mu,x)} = \frac{nb_n\cdot |x-\mu|}{2V(x)} e^{-nb_n\cdot \kl(\mu,x)},
\end{align}
where $(\kl(\mu,x))'$ denotes the derivative of $\kl(\mu,x)$ with respect to $x$, and $b_n$ is a function of $n$ and will be chosen later.  
 
We assume the reward is supported in $[R_{\min},R_{\max}]$. Note that $R_{\min}=0$, and $R_{\max}=1$ for Bernoulli rewards, and $R_{\min}=-\infty$, and $R_{\max}=\infty$ for Gaussian rewards. Let $p(x)$ and $q(x)$ be the density functions of two exponential family distributions with mean values $\mu_p$ and $\mu_q$ respectively. By the definition in Section \ref{sec:preliminary}, we have $\kl(\mu_p,\mu_q)=\text{KL}(p(x),q(x))=\int_{R_{\min}}^{R_{\max}} p(x) \log \frac{p(x)}{q(x)}\dd x$. 
\begin{proposition}
If the mean reward of $q(x)$ is equal to the maximum value in its support, i.e.,  $\mu_{q}=R_{\max}$, we will have $\kl(\mu,R_{\max})=\infty$ for any $\mu<R_{\max}$.
\end{proposition}
\begin{proof}
First consider the case that $R_{\max}<\infty$. Since the mean value concentrates on the maximum value, we must have $q(x)=0$ for all $x<R_{\max}$, which immediately implies $\kl(\mu,R_{\max})=\infty$ for any $\mu<R_{\max}$. For the case that $R_{\max}=\infty$, from \eqref{eq:pinsk} and the assumption that $V<\infty$, we also have $\kl(\mu,\infty)=(\infty-\mu)^2/V=\infty$.
\end{proof}

Similarly, we can also prove that $\kl(\mu,R_{\min})=\infty$ for $\mu>R_{\min}$. Based on these properties, we can easily verify that a sample from the proposed sampling distribution $\theta\sim\cP$ has the following tail bounds: for $z\in [\mu,R_{\max})$, it holds that
\begin{align}\label{eq:perpty1}
\small
\PP(\theta\geq z)&=\int_{z}^{R_{\max}} f(x;\mu,n) \dd x =- 1/2e^{-nb_n\cdot \kl(\mu,x)} \bigg|^{R_{\max}}_{z} 
= 1/2e^{-nb_n\cdot \kl(\mu,z)},
\end{align} 
and for $z\in (R_{\min},\mu]$, it holds that
\begin{align}\label{eq:perpty2}
\small
\PP(\theta\leq z)&=\int_{R_{\min}}^{z} f(x;\mu,n) \dd x =1/2e^{-nb_n\cdot \kl(\mu,x)} \bigg|_{R_{\min}}^{z} 
= 1/2e^{-nb_n\cdot \kl(\mu,z)}.
\end{align} 
Note that  $\int_{R_{\min}}^{R_{\max}}f(x;\mu,n) \dd x=\int_{R_{\min}}^{\mu} f(x;\mu,n) \dd x+ \int_{\mu}^{R_{\max}}f(x;\mu,n) \dd x =1$, which indicates the PDF of $\cP$ is well-defined. 

\noindent \textbf{Intuition for the Design of the Sampling Distribution.}  The tail bounds in \eqref{eq:perpty1} and \eqref{eq:perpty2} provide proper anti-concentration and  concentration bounds for the sampling distribution $\cP$ as long as we have corresponding lower and upper bounds of $e^{-nb_n\cdot \kl(\mu,z)}$. When $n$ is large, we will choose $b_n$ to be close to $0$, and thus \eqref{eq:perpty1} and \eqref{eq:perpty2}
ensure that the sample of the corresponding arm concentrates in the interval $(\mu-\epsilon,\mu+\epsilon)$ with an exponentially small probability $e^{-\kl(\mu-\epsilon,\mu+\epsilon)}$, which is crucial for achieving a tight finite-time regret.

\noindent \textbf{How to Sample from $\cP$.} We show that sampling from $\cP$ is tractable since the CDF of $\cP$ is invertible. And thus we do not need to use approximate sampling methods such as Monte Carlo Markov Chain and Hastings-Metropolis \citep{korda2013thompson}. In particular, according to \eqref{eq:perpty1} and  \eqref{eq:perpty2}, the CDF of $\cP(\mu,n)$ is 
\begin{align*}
F(x)=
    \begin{cases}
    1- 1/2e^{-nb_n\cdot \kl(\mu,x)} &  x\geq \mu,\\
    1/2e^{-nb_n\cdot \kl(\mu,x)} & x\leq \mu.
    \end{cases}
\end{align*}
To sample from $\cP(\mu,n)$, we can first pick $y$ uniformly random from $[0,1]$. Then, for $y\geq 1/2$, we solve the equation $y=1-1/2e^{-nb_n\cdot \kl(\mu,x)}$ for $x$ ($x\geq \mu$), which is equivalent to solving $\log (1/(2(1-y)))/(nb_n)=\kl(\mu,x)$. For $y\leq 1/2$, we solve the equation $y=1/2e^{-nb_n\cdot \kl(\mu,x)}$ for $x$ ($x\leq \mu$), which is equivalent to solving $\log (1/(2y))/(nb_n)=\kl(\mu,x)$. In this way, $x$ would be an exact sample from distribution $\cP$.

\begin{table*}[t]
{\footnotesize
    \centering
    \caption{Comparisons of different  algorithms on $K$-armed bandits over a horizon $T$. For any algorithm, \emph{Asym. Opt} is the indicator whether it is asymptotically optimal, \emph{minimax ratio} is the scaling of its worst-case regret w.r.t. the minimax optimal regret $O(\sqrt{VKT})$, \emph{sub-UCB} indicates whether it satisfies the sub-UCB criteria, and \emph{Anytime} indicates whether it needs the knowledge of the horizon length $T$ in advance.  \label{table:regret_comparison}} 
    \begin{tabular}{ccccccl}
    \toprule
    \multirow{2}{*}{Algorithm}  & \multirow{2}{*}{Reward Type}&  Asym. &  \multicolumn{2}{c}{Finite-Time Regret} & \multirow{2}{*}{Anytime} & \multirow{2}{*}{References}\\
    \cline{4-5}
    & & Opt &  Minimax Ratio  &Sub-UCB & & \\
    \midrule
    MOSS & $[0,1]$ & no & 1 & no & no  & \citep{audibert2009minimax} \\
    Anytime MOSS  & $[0,1]$ & no & 1  & no & yes &  \citep{degenne2016anytime} \\
    KL-UCB$^{++}$ & Exponential Family  & yes & 1 & no &no & \citep{menard2017minimax} \\
    OCUCB   & sub-Gaussian   & no & $\sqrt{\log \log T}$ & yes &yes  &\citep{lattimore2016regret} \\
    AdaUCB  & Gaussian & yes & 1 & yes & no & \citep{lattimore2018refining} \\
    MS & sub-Gaussian & yes & $\sqrt{\log K}$  & yes &yes &\citep{bian2021maillard}\\
    $\algname$ & Exponential Family &  yes& $\sqrt{\log K}$  & yes & yes &This paper \\
    $\algname^{+}$ & Exponential Family &  yes& 1  & no & yes & This paper \\
    \bottomrule
    \end{tabular}
}
\end{table*}

\subsection{Regret Analysis of $\algname$ for Exponential Family Rewards}
Now we present the regret bound of $\algname$ for general exponential family bandits. The sampling distribution used in Algorithm \ref{alg:rots} is defined in \eqref{def:general_posterior}.
\begin{theorem}\label{thm:regret_rots}
Let $b_n=(n-1)/n$. Let $\cP$ be the sampling distribution defined in Section \ref{subsec:expon}. The regret of  
Algorithm \ref{alg:rots} satisfies the following finite-time bounds: 
\begin{align}
    R_{\mu}(T)&=\sum_{i\in [K]:\Delta_i>\lambda} O\bigg(\Delta_i+\frac{V\log(T\Delta_i^2/V)}{\Delta_i} \bigg)+\max_{i\in [K],\Delta_i\leq \lambda} \Delta_i \cdot T,\label{eq:regret-expts-1}\\
    R_{\mu}(T)&= O\bigg(\sum_{i=2}^{K}\Delta_i+\sqrt{VKT\log K} \bigg),\label{eq:regret-expts-2}
\end{align}
where $\lambda \geq 16\sqrt{V/T}$, and also satisfies the following asymptotic bound simultaneously:
\begin{align}
\label{eq:regret-expts-3}
\lim_{T\rightarrow \infty}\frac{R_{\mu}(T)}{\log T}= \sum_{i=2}^{K}\frac{\Delta_i}{\kl(\mu_i,\mu_1)}.
\end{align}
\end{theorem}
\begin{remark}
Similar to the argument by \citet{auer2010ucb}, we can see that the logarithm term in \eqref{eq:regret-expts-1} is the main term for suitable $\lambda$. For instance, if we choose $\lambda=16\sqrt{V/T}$, we will have $\max_{i\in [K],\Delta_i\leq \lambda}\Delta_i T\leq\sqrt{VT}$, which is in the order of $O(V/\Delta_i)$ due to $\Delta_i\leq\lambda$. Thus it is obvious to see that the regret in \eqref{eq:regret-expts-1} satisfies the sub-UCB criteria.
\end{remark}

It is worth highlighting that ExpTS is an anytime algorithm and simultaneously satisfies the sub-UCB criteria in \eqref{def:problem-dependent-regret}, the minimax optimal regret up to a factor $\sqrt{\log K}$, and the asymptotically optimal regret.  ExpTS is also the first Thompson sampling algorithm that provides finite-time regret bounds for exponential family of rewards.  Compared with state-of-the-art MAB algorithms listed in Table \ref{table:regret_comparison}, ExpTS is comparable to the best known UCB algorithms and no algorithms can dominate ExpTS.  In particular, compared with MS \citep{bian2021maillard} and OCUCB \citep{lattimore2016regret},  ExpTS is asymptotically optimal for exponential family of rewards, while MS is only asymptotically optimal for sub-Gaussian rewards and OCUCB is not asymptotically optimal. We note that Exponential Family does not cover the sub-Gaussian rewards. However, since we only use the tail bound to approximate the reward distribution, ExpTS can  also be extended to solve sub-Gaussian reward bandits, which we leave as a future open direction.

\subsection{Gaussian and Bernoulli Reward Distributions} 
The choice of $\cP$ in \eqref{def:general_posterior} seems complicated for a general exponential family reward distribution, even though we only need the sampling distribution to satisfy a nice tail bound derived from this reward distribution. Nevertheless, when the reward distribution has a closed form such as Gaussian and Bernoulli distributions, we can directly use the posterior in standard Thompson sampling and obtain the asymptotic and finite-time regrets in the previous section. 
\begin{theorem}
\label{thm:Gaussian}
If the reward follows a Gaussian distribution with a known variance $V$, we can set the sampling distribution in Algorithm \ref{alg:rots} as $\cN(\hmu_{i}(t),V/T_{i}(t))$. The resulting algorithm (denoted as Gaussian-TS) enjoys the same regret bounds presented in Theorem \ref{thm:regret_rots}.
\end{theorem}
\begin{remark}
\citet{jin2021mots} shows in Lemma 5 of their paper that for Gaussian rewards, Gaussian-TS has a regret bound at least $\Omega(\sqrt{VKT\log T})$ if the standard regret decomposition discussed in Section~\ref{sec:challenge} is adopted in the proof \citep{agrawal2017near,lattimore2018bandit,jin2021mots}. With our new regret decomposition and the lower confidence bound introduced in \eqref{eq:challenge-lows}, we improve the worst-case regret of Gaussian-TS for Gaussian rewards to $O(\sqrt{VKT\log K})$. 

\citet{jin2021mots} also shows that their algorithms MOTS/MOTS-$\cJ$ can overcome the under-estimation issue of \eqref{eq:challenge0}. However, they are either at the cost of sacrificing the asymptotic optimality or not generalizable to exponential family bandits. In specific, (1) For Gaussian rewards, MOTS \citep{jin2021mots} enlarges the variance of Gaussian posterior by a factor of $1/\rho$, where $\rho\in (0,1)$, which loses the asymptotic optimality by a factor of $1/\rho$ resultantly. (2) For Gaussian rewards, MOTS-$\cJ$  \citep{jin2021mots} introduces the  Rayleigh posterior to overcome the under-estimation while maintaining the asymptotic optimality. However, it is not clear whether the idea can be generalized to exponential family rewards. Interestingly, their experimental results show that compared with Rayleigh posterior, Gaussian posterior actually has a smaller regret empirically. Therefore, to use a Gaussian sampling distribution, the new regret decomposition and the novel lower confidence bound in our paper is a better way to overcome the under-estimation issue of Gaussian-TS. 
\end{remark}

\begin{theorem}
\label{thm:Bernoulli}
If the reward distribution is Bernoulli, we can set the sampling distribution $\cP$ in Algorithm \ref{alg:rots} as Beta posterior  $\cB(S_{i}(t)+1, T_{i}(t)-S_{i}(t)+1)$, where $S_{i}(t)$ is the number of successes among the $T_{i}(t)$ plays of arm $i$. We denote the resulting algorithm as Bernoulli-TS, which enjoys the same regret bounds as in Theorem \ref{thm:regret_rots}.
\end{theorem}
\begin{remark}
\citet{agrawal2017near} proved that for Bernoulli rewards, Thompson sampling with Beta posterior  is asymptotically optimal and has a worst-case regret in the order of $O(\sqrt{KT \log T})$. Our regret analysis improve the worst-case regret to $O(\sqrt{KT \log K})$. Besides, \citet{agrawal2017near} proved that Gaussian-TS applied to the Bernoulli reward setting has a regret $O(\sqrt{KT\log K})$. However, no asymptotic regret was guaranteed in this setting.
\end{remark}

\section{Minimax Optimal Thompson Sampling for Exponential Family Rewards}
In this section, in order to remove the extra logarithm term in the worst-case regret of \algname, we introduce a new sampling distribution that adds a greedy exploration step to the sampling distribution used in \algname. Specifically, the new algorithm $\algname^+$ is the same as $\algname$ but uses a new sampling distribution $\cP^+(\mu,n)$. A sample $\theta$ is generated from $\cP^+(\mu,n)$ in the following way: $\theta=\mu$ with probability $1-1/K$ and  $\theta\sim \cP(\mu,n)$ with probability $1/K$. 

\noindent\textbf{Over-Estimation of Sub-Optimal Arms.} We first elaborate the over-estimation issue of sub-optimal arms, which results in the extra $\sqrt{\log K}$ term in the worst-case regret of Thompson sampling. To explain, suppose that the sample of each arm $i$ has a probability $p=\PP(\theta_i(t)\geq \theta_{1}(t))$ to become larger than the  sample of arm 1. Note that when this event happens, the algorithm chooses the wrong arm and thus incurs a regret. Intuitively, the probability of making a mistake will be $K-1$ times larger due to the union bound over $K-1$ sub-optimal arms, which leads to an additional $\sqrt{\log K}$ factor in the worst-case regret.  To reduce the probability $\PP(\theta_i(t)\geq \theta_{1}(t))$, ExpTS$^+$ adds a greedy step that chooses the ExpTS sample with probability $1/K$ and chooses the arm with the largest empirical average reward with probability $1-1/K$. Then we can prove that for sufficiently large $s$, with high probability we have $\hat{\mu}_{is}<\theta_1(t)$ and in this case it holds that $\PP(\theta_{i}(t)\geq \theta_1(t))= p/K$. Thus the extra factor $\sqrt{\log K}$ in regret is removed.

In specific, we have the following theorem showing that $\algname^+$ is asymptotically optimal and minimax optimal simultaneously.
\begin{theorem}\label{thm:expTS_plus_minimax}
Let $b_n=(n-1)/n$. The regret of $\algname^+$ satisfies
\begin{align*}
    R_{\mu}(T)=O\bigg(\sum_{i=2}^K\Delta_i+\sqrt{VKT}\bigg), 
  \quad\text{and} \quad  \lim_{T\rightarrow \infty}\frac{R_{\mu}(T)}{\log T}=\sum_{i=2}^K\frac{\Delta_i}{\kl(\mu_i,\mu_1)}.
\end{align*}
\end{theorem}
This is the first time that the Thompson sampling algorithm achieves the minimax and asymptotically optimal regret for exponential family of reward distributions. Moreover, $\algname^+$ is also an anytime algorithm since it does not need to know the horizon $T$ in advance. 
\begin{remark}[Sub-Gaussian Rewards]
In the proof of Theorem \ref{thm:expTS_plus_minimax}, we do not need the strict form of the PDF of the empirical mean reward  $\hmu_{is}$, but only need the maximal inequality (Lemma \ref{lem:maximal-inequality}). This means that the proof can be straightforwardly extended to sub-Gaussian reward distributions, where similar maximal inequality  holds \citep{jin2021double}. 

It is worth noting that MOTS proposed by \citet{jin2021mots} (Thompson sampling with a clipped Gaussian posterior) also achieves the minimax optimal regret for sub-Gaussian rewards, but it can not keep the asymptotic optimality simultaneously with the same algorithm parameters. In particular, to achieve the minimax optimality, MOTS will have an additional $1/\rho$ factor in the asymptotic regret with $0<\rho<1$. Moreover, different from $\algname^+$, MOTS is only designed for fixed $T$ setting and thus is not an anytime algorithm. 
\end{remark}

\begin{remark}[Gaussian and Bernoulli Rewards]
Following the idea of \algname$^+$, we can easily obtain new algorithms Gaussian-TS$^+$ and Bernoulli-TS$^+$ for Gaussian and Bernoulli rewards. Using a similar proof, we can show that they are simultaneously minimax and asymptotically optimal.
\end{remark}

\section{Conclusions}
We studied Thompson sampling for exponential family of reward distributions. We proposed the ExpTS algorithm and proved it satisfies the sub-UCB criteria for problem-dependent finite-time regret, as well as achieves the asymptotic optimality and the minimax optimality up to a factor of $\sqrt{\log K}$ for exponential family rewards. Furthermore, we proposed a  variant of $\algname$, dubbed \algname$^+$, that adds a greedy exploration step to balance between the sample generated in ExpTS and the empirical mean reward for each arm. We proved that $\algname^+$ is simultaneously minimax and asymptotically optimal. We also extended our proof techniques to standard Thompson sampling with common posterior distributions and improved existing results. For Gaussian/Bernoulli rewards, we proved Gaussian-TS/Bernoulli-TS satisfies the sub-UCB criteria and enjoys the asymptotically optimal regret and minimax optimal regret $O(\sqrt{KT\log K})$. Although \algname$^+$ is simultaneously minimax and asymptotically optimal, it does not satisfies the sub-UCB criteria anymore. It would be an interesting future direction to design a sampling distribution such that TS is asymptotically optimal, minimax optimal, and matches the sub-UCB criteria at the same time.

\appendix

\section{Proof of the Finite-Time Regret Bound of $\algname$}\label{sec:proof_finite}
In this section, we prove the finite-time regret bound of $\algname$ presented in Theorem \ref{thm:regret_rots}. Specifically, we prove the sub-UCB property of $\algname$ in \eqref{eq:regret-expts-1} and the nearly minimax optimal regret of $\algname$ in \eqref{eq:regret-expts-2}.

\subsection{Proof of the Main Results}
We first focus on bounding the number of pulls of arm $i$ for the case that $\Delta_i> 16\sqrt{V/T}$. We start with the decomposition. Note that due to the warm start of Algorithm \ref{alg:rots}, each arm has been pulled once in the first $K$ steps. For any $\epsilon>8\sqrt{V/T}$, define event $E_{i,\epsilon}(t)=\{\theta_i(t)\leq \mu_1-\epsilon \}$, for all $i\in[K]$, which indicates that the estimate of arm $i$ at time step $t$ is smaller than the lower bound of the true mean reward of arm $1$ ($\mu_1-\epsilon\leq\mu_1$). The expected number of times that Algorithm~\ref{alg:rots} plays arms $i$ is bounded as follows.
\begin{align}\label{eq:decomp_num_pull_i}
    \mathbb{E}[T_i(T)]
    & = 1+ \EE\left[\sum_{t=K+1}^T \ind \{A_t=i, E_{i,\epsilon} (t) \}+ \sum_{t=K+1}^T \ind \{A_t=i, E_{i,\epsilon}^c (t) \} \right] \notag \\
    & = 1+ \underbrace{\EE\left[ \sum_{t=K+1}^T \ind \{A_t=i, E_{i,\epsilon} (t) \} \right]}_{A}   +\underbrace{\EE\left[ \sum_{t=K+1}^T \ind \{A_t=i, E_{i,\epsilon}^c (t) \} \right]}_{B},
\end{align}
where $E^c$ is the complement of an event $E$, $\epsilon>8\sqrt{V/T}$ is an arbitrary constant, and we used the fact $T_i(T)=\sum_{t=1}^{T}\ind\{A_t=i\}$. In what follows, we bound these terms individually.

\paragraph{Bounding Term $A$:}
Let us define
\begin{align}
\label{eq:alphas}
    \alpha_s=\sup_{x\in [0,\mu_1-\epsilon-R_{\min})} \kl(\mu_1-\epsilon-x,\mu_1)\leq 4\log(T/s)/s.
\end{align}
We decompose the term $\EE\left[ \sum_{t=K+1}^T \ind \{A_t=i, E_{i,\epsilon} (t) \} \right]$ by the following lemma.
\begin{lemma}
\label{lem:underest-decomp}
Let  $M=\lceil 16V\log(T\epsilon^2/V)/\epsilon^2 \rceil$ and $\alpha_s$ be the same as defined in \eqref{eq:alphas}. Then,
\begin{align*}
   \EE\left[ \sum_{t=K+1}^T \ind \{A_t=i, E_{i,\epsilon} (t) \} \right]& \leq  \sum_{s=1}^{M}\mathbb{E}\Bigg[\left(\frac{1}{G_{1s}(\epsilon)}-1\right) \cdot \ind\{\hmu_{1s}\in L_s \} \Bigg]+\Theta\bigg(\frac{V}{\epsilon^2}\bigg), 
\end{align*}
where $G_{is}(\epsilon)=1-F_{is}(\mu_1-\epsilon)$, $F_{is}$ is the CDF of $\cP(\hmu_{is},s)$, and   $L_s=\big(\mu_1-\epsilon-\alpha_s, R_{\max} \big]$. 
\end{lemma}

The first term on the right hand side could be further bounded as follows.
\begin{lemma}\label{lem:minimax-undetestimation-rots}
Let $M$, $G_{1s}(\epsilon)$, and $L_s$ be the same as defined in Lemma \ref{lem:underest-decomp}. Then it holds that
\begin{align*}
\sum_{s=1}^{M}\EE_{\hmu_{1s}} \Bigg[\bigg(\frac{1}{G_{1s}(\epsilon)} \bigg) \cdot \ind\{\hmu_{1s}\in L_s \} \Bigg]=O\bigg(\frac{V\log(T\epsilon^2/V)}{\epsilon^2} \bigg).
\end{align*}
\end{lemma}

Combining Lemma \ref{lem:underest-decomp} and Lemma \ref{lem:minimax-undetestimation-rots} together, we have the upper bound of term $A$ in \eqref{eq:decomp_num_pull_i}.
\begin{align*}
  A=O\bigg(\frac{V\log(T\epsilon^2/V)}{\epsilon^2} \bigg).
\end{align*}

\paragraph{Bounding Term $B$:}
To bound the second term in \eqref{eq:decomp_num_pull_i}, we first prove the following lemma that bounds the number of time steps when the empirical average reward of arm $i$ deviates from its mean value.
\label{sec:boundingB}
\begin{lemma}
\label{lem:mini-overest-1}
Let $N=\min\{1/(1-\kl(\mu_i+\rho_i,\mu_1-\epsilon)/
\log(T\epsilon^2/V)),2\}$. For any $\rho_i, \epsilon>0$ that satisfies $\epsilon+\rho_i<\Delta_i$,  then
\begin{align*}
 \EE\left[\sum_{t=K+1}^{T} \ind\{A_t=i, E_{i,\epsilon}^c(t) \} \right] \leq 1+\frac{2V}{\rho_i^2}+\frac{V}{\epsilon^2}+\frac{N\log(T\epsilon^2/V)}{\kl(\mu_i+\rho_i,\mu_1-\epsilon)}.
\end{align*}
\end{lemma}

Applying Lemma \ref{lem:mini-overest-1}, we have the following bound for term $B$ in \eqref{eq:decomp_num_pull_i}.
\begin{align*}
    \EE\left[\sum_{t=K+1}^{T} \ind\{A_t=i, E_{i,\epsilon}^c(t) \} \right] & \leq 1+\frac{2V}{\rho_i^2}+\frac{V}{\epsilon^2}+\frac{N\log(T\epsilon^2/V)}{\kl(\mu_i+\rho_i,\mu_1-\epsilon)}\notag \\
    & \leq 1+\frac{2V}{\rho_i^2}+\frac{V}{\epsilon^2}+\frac{4V\log(T\epsilon^2/V)}{(\Delta_i-\epsilon-\rho_i)^2},
\end{align*}
where the last inequality is due to \eqref{eq:pinsk} and $N\leq 2$.

\paragraph{Putting It Together:}
Substituting the  bounds of terms $A$ and $B$ back into \eqref{eq:decomp_num_pull_i}, we have
\begin{align*}
    \EE[T_{i}(T)]=O\bigg(1+\frac{V\log(T\epsilon^2/V)}{(\Delta_i-\epsilon-\rho_i)^2} +\frac{V}{\rho_i^2}+\frac{V\log(T\epsilon^2/V)}{\epsilon^2}\bigg).
\end{align*}
Let $\epsilon=\rho_i=\Delta_i/4$, we have
\begin{align*}
    \EE[T_{i}(T)]=O\bigg(1+\frac{V\log(T\Delta_i^2/V)}{\Delta_i^2}\bigg).
\end{align*}
Note that we have assumed $\Delta_i>16\sqrt{V/T}$ at the beginning of the proof. Therefore, the total regret can be decomposed as follows,
\begin{align*}
    R_{\mu}(T)=\sum_{i\in [K]:\Delta_i>\lambda} O\bigg(\Delta_i+\frac{V\log(T\Delta_i^2/V)}{\Delta_i} \bigg)+\max_{i\in [K],\Delta_i\leq \lambda} \Delta_i \cdot T,
\end{align*}
for any $\lambda\geq 16\sqrt{V/T}$.
By choosing $\lambda=16\sqrt{VK\log K/T}$, we obtain the following worst-case regret: $R_{\mu}(T)=O(\sqrt{VKT\log K})$. This completes the proof of the finite-time regret bounds of $\algname$.

\subsection{Proof of Supporting Lemmas}
In this subsection, we prove the lemmas used in the proof of our main results in this section.
\subsubsection{Proof of Lemma \ref{lem:underest-decomp}}
Define $\cE$ to be the event such that  $\hmu_{1s}\in L_s$ holds for all $s\in [T]$. The proof of Lemma \ref{lem:underest-decomp} needs the following lemma, which is used for bounding $\PP(\cE^c)$.

\begin{lemma}
\label{lem:book-minimax}
Let $\epsilon>0$, $b\in[K]$ and $f(\epsilon)=\lceil 16V\log(T\epsilon^2/(bV))/\epsilon^2  \rceil$. Assume $T\geq bf(\epsilon)$. Then,
\begin{align*}
\PP \big(\exists 1\leq s \leq f(\epsilon): \hat{\mu}_{1s}\leq \mu_1-\epsilon, \kl(\hat{\mu}_{1s},\mu_1)\geq 4\log(T/(bs))/s \big) \leq \Theta\bigg(\frac{bV}{T\epsilon^2}\bigg).
\end{align*}
\end{lemma}
The proof of Lemma \ref{lem:book-minimax} could be found in Section \ref{sec:technical-lem}. Now, we are ready to prove Lemma \ref{lem:underest-decomp}.
\begin{proof}[Proof of Lemma \ref{lem:underest-decomp}]
The indicator function can be decomposed based on $\cE$, that is
\begin{align}
\label{eq:mini-decom-f10}
    &\EE\left[ \sum_{t=K+1}^T \ind \{A_t=i, E_{i,\epsilon} (t)\} \right] \notag\\
    & \leq T\cdot \PP(\cE^c)+  \EE\left[ \sum_{t=K+1}^T \big[\ind \{A_t=i, E_{i,\epsilon} (t) \}\cdot \ind\{\hmu_{1T_{i}(t-1)}\in L_{T_{i}(t-1)}\} \big] \right] \notag \\
   &  \leq \Theta\bigg(\frac{V
   }{T\epsilon^2}\bigg)+  \EE\left[ \sum_{t=K+1}^T \big[\ind \{A_t=i, E_{i,\epsilon} (t) \}\cdot \ind\{\hmu_{1T_{i}(t-1)}\in L_{T_{i}(t-1)}\} \big] \right],
\end{align}
where the second inequality is due to Lemma \ref{lem:book-minimax} with $b=1$ and from the fact $\epsilon>8\sqrt{V/T}$, $T\geq f(\epsilon)$.
 Let $\cF_{t}=\sigma(A_1,r_1,\cdots,A_{t},r_t)$ be the filtration. Note that $\theta_i(t)$ is sampled from $\cP(\hat\mu_{i}(t-1),  T_i(t-1))$. Recall the definition, we know that $\hat\mu_i(t-1)=\hat\mu_{is}$ as long as $s=T_i(t-1)$. By the definition of $G_{is}(x)$, it holds that
\begin{align}
\label{eq:mini-decom-10}
    G_{1T_{1}(t-1)}(\epsilon)=\PP(\theta_{1}(t)\geq \mu_1-\epsilon \mid \cF_{t-1}).
\end{align}
Consider the following two cases. \\
\textbf{Case 1: $t:T_1(t-1)\leq M$.}
The proof of this case is similar to that of \citep[Theorem 36.2]{lattimore2018bandit}.  Let $A'_{t}=\arg \max_{i\neq 1}\theta_{i}(t)$. Then 
\begin{align}
\label{eq:mini-decom-2}
    \PP(A_{t}=1\mid \cF_{t-1}) & \geq \PP(\{\theta_{1}(t)\geq \mu_1-\epsilon\}\cap\big\{ A'_{t}=i, E_{i,\epsilon} (t) \big\}\mid \cF_{t-1}) \notag \\
    & = \PP(\theta_{1}(t)\geq \mu_1-\epsilon \mid \cF_{t-1}) \cdot  \PP(A'_{t}=i,E_{i,\epsilon} (t)\mid \cF_{t-1}) \notag \\
    & \geq \frac{G_{1T_{1}(t-1)}(\epsilon)}{1-G_{1T_{1}(t-1)}(\epsilon)} \cdot  \PP(A_{t}=i,E_{i,\epsilon} (t)\mid \cF_{t-1}),
\end{align}
where the first inequality is due to the fact that when both event $\{\theta_{1}(t)\geq \mu_1-\epsilon\}$ and event $\{A'_{t}=i,E_{i,\epsilon} (t) \}$ hold, we must have $\{A_t=1\}$, the first equality is due to $\theta_1(t)$ is conditionally independent of $A_t'$ and $E_{i,\epsilon} (t)$ given $\cF_{t-1}$,   and  the last inequality is from \eqref{eq:mini-decom-10} and the fact that 
\begin{align*}
    \PP(A_{t}=i,E_{i,\epsilon} (t)\mid \cF_{t-1}) \leq (1-\PP(\theta_{1}(t)\geq \mu_1-\epsilon \mid \cF_{t-1}))\cdot  \PP(A_{t}'=i,E_{i,\epsilon} (t)\mid \cF_{t-1}),
\end{align*}
which is due to $\{ A_t=i, E_{i,\epsilon} (t) \ \text{occurs}\}\subseteq \{ A'_{t}=i, E_{i,\epsilon} (t) \ \text{occurs} \}\cap \{\theta_{1}(t)\leq \mu_1-\epsilon\}$ and the two intersected events are conditionally independent given $\cF_{t-1}$. 
Therefore,  we have
\begin{align}
    \label{eq:decom-T1t<M0}
    \EE \left[\sum_{t:T_{1}(t-1)\leq M} \ind \{A_t=i,E_{i,\epsilon}(t) \}
    \right] & \leq \EE \left[ \sum_{t:T_{1}(t-1)\leq M}\bigg(\frac{1}{G_{1T_{1}(t-1)}(\epsilon)}-1 \bigg)\PP(A_t=1 \mid \cF_{t-1})\right]\notag \\
    & \leq \EE \left[ \sum_{t:T_{1}(t-1)\leq M}\bigg(\frac{1}{G_{1T_{1}(t-1)}(\epsilon)}-1 \bigg)\ind\{ A_t=1\}\right] \notag \\
    &\leq \EE \left[\sum_{s=1}^{M} \bigg(\frac{1}{G_{1s}(\epsilon)}-1 \bigg) \right],
\end{align}
where the first inequality is from \eqref{eq:mini-decom-2}. \\
\textbf{Case 2: $t:T-1 \geq T_1(t-1)> M$.}  For this case, we have
\begin{align}
\label{eq:decom-T1t>M0}
\EE  \left[\sum_{t:T_{1}(t-1)> M}^{T-1} \ind \{A_t=i,E_{i,\epsilon}(t) \}\right] & \leq \EE \left[\sum_{t:T_{1}(t-1)> M}^T \ind \{\theta_{1}(t)< \mu_1-\epsilon\}\right] \notag \\
& \leq T\cdot \PP \big(\exists s> M:\hmu_{1s}<\mu_1-\epsilon/2\big) \notag \\
&\quad +\EE \left[\sum_{t:T_{1}(t-1)> M} \ind \{\theta_{1}(t)< \mu_1-\epsilon \mid \hmu_{1T_{1}(t-1)}\geq \mu_1-\epsilon/2\}\right] \notag \\
& \leq T\cdot e^{-M(\mu_1-(\mu_1-\epsilon/2))^2/(2V)} \notag \\
& \qquad +\sum_{t:T_{1}(t-1)>M}\PP \big[\theta_{1}(t)<\mu_1-\epsilon \mid \hmu_{1T_{1}(t-1)}\geq \mu_1-\epsilon/2   \big] \notag \\
& \leq \frac{V}{\epsilon^2}+ T\cdot e^{-Mb_M\epsilon^2/(8V)}\notag\\
&\leq \frac{2V}{\epsilon^2},
\end{align}
where the first inequality is due to the fact that $\{A_t=i,E_{i,\epsilon}(t)\}\subseteq \{ \theta_{1}(t)<\mu_1-\epsilon\}$,  the third inequality is due to Lemma \ref{lem:maximal-inequality},  the fourth inequality is due to \eqref{eq:pinsk}, and the last inequality is due to the fact $b_{M}\geq 1/2$. Combining \eqref{eq:mini-decom-f10}, \eqref{eq:decom-T1t<M0}, and \eqref{eq:decom-T1t>M0} together, we finish the proof of Lemma \ref{lem:underest-decomp}. 
\end{proof}

Note that in order to bound term $A$, we need the following lemma that states the upper bound of the first term in Lemma \ref{lem:underest-decomp}.

\subsubsection{Proof of Lemma \ref{lem:minimax-undetestimation-rots}}
Let $p(x)$ be the PDF of $\hmu_{1s}$ and $\theta_{1s}$ be a sample from $\cP(\hmu_{1s},s)$.  We have
\begin{align}
\label{eq:finite-underest-0}
   &\sum_{s=1}^{M}\EE_{\hmu_{1s}} \Bigg[\bigg(\frac{1}{G_{1s}(\epsilon)} -1\bigg) \cdot \ind\{\hmu_{1s}\in L_s \} \Bigg]\notag\\
   &\leq  \underbrace{\sum_{s=1}^{M}\bigg(\int_{\mu_1-\epsilon/2}^{R_{\max}} p(x)/  \PP(\theta_{1s}\geq \mu_1-\epsilon \mid \hmu_{1s}=x)\dd x-1\bigg)}_{A_1}\notag \\
   &\qquad +\underbrace{\sum_{s=1}^{M}\int_{\mu_1-\epsilon}^{\mu_1-\epsilon/2} p(x)/  \PP(\theta_{1s}\geq \mu_1-\epsilon \mid \hmu_{1s}=x)\dd x}_{A_2}\notag\\ &\qquad+\underbrace{\sum_{s=1}^{M}\int_{\mu_1-\epsilon-\alpha_s}^{\mu_1-\epsilon}\Big[ p(x)/  \PP(\theta_{1s}\geq \mu_1-\epsilon \mid \hmu_{1s}=x) \Big] \dd x}_{A_3},
\end{align}
where the inequality is due to the definition of $L_s$\footnote{For the discrete reward distribution, we can use the Dirac delta function for the integral.}. \\
\textbf{Bounding term $A_1$.}
For term $A_1$, we divide $\sum_{s=1}^{M}$ into two term, i.e., $\sum_{s=1}^{\lfloor 32V/\epsilon^2 \rfloor}$ and $\sum^{M}_{s=\lceil 32V/\epsilon^2 \rceil}$. Intuitively, for $s\geq 32V/\epsilon^2$, $\PP(\theta_{1s}\geq \mu_1-\epsilon\mid \hat{\mu}_{1s}\geq \mu_1-\epsilon/2)$ will be large. We have
\begin{align}
\label{bounding-a_1}
  A_1 &= \sum_{s=1}^{M}\bigg(\int_{\mu_1-\epsilon/2}^{R_{\max}} \frac{p(x)}  {\PP(\theta_{1s}\geq \mu_1-\epsilon \mid \hmu_{1s}=x)} \dd x -1\bigg) \notag \\
  &\leq \frac{32V}{\epsilon^2}+\sum_{s=\lceil 32V/\epsilon^2 \rceil}^{M}\bigg(\int_{\mu_1-\epsilon/2}^{R_{\max}} \frac{p(x)}  {\PP(\theta_{1s}\geq \mu_1-\epsilon \mid \hmu_{1s}=x)} \dd x -1\bigg) \notag \\
   &\leq \frac{32V}{\epsilon^2} +\sum_{s=\lceil 32V/\epsilon^2 \rceil }^{M}\bigg( \frac{1}{1-e^{-s/2\cdot\kl(\mu_1-\epsilon/2,\mu_1-\epsilon)}}-1\bigg)\notag \\
   &\leq \frac{32V}{\epsilon^2}+\sum_{s=\lceil 32V/\epsilon^2\rceil }^{M} \bigg(\frac{1}{1-e^{-s\epsilon^2/(16V)}}-1\bigg) \notag \\
   & =\frac{16V}{\epsilon^2}+ \sum_{s=\lceil {32V}/{\epsilon^2}\rceil }^{M} \frac{1}{e^{s\epsilon^2/(16V)}-1}\notag\\
   &\leq \frac{16V}{\epsilon^2}+ \frac{16V}{\epsilon^2}\sum_{s=1}^{\infty} \frac{1}{e^{1+s}-1} \notag \\
   &\leq \frac{32V}{\epsilon^2}.
\end{align}
For the first inequality, we used the fact $\PP(\theta_{1s}\geq \mu_1-\epsilon\mid \hat{\mu}_{1s}\geq \mu_1-\epsilon)\geq 1/2$, which is due to \eqref{eq:perpty1}. The second inequality is due to \eqref{eq:perpty2} and the fact $b_s\geq 1/2$. The third inequality is due to \eqref{eq:pinsk}. \\
\textbf{Bounding term $A_2$.} 
We have 
\begin{align}
\label{boundinga_2}
  A_2 = \sum_{s=1}^{M}\int^{\mu_1-\epsilon/2}_{\mu_1-\epsilon} \frac{p(x)}  {\PP(\theta_{1s}\geq \mu_1-\epsilon \mid \hmu_{1s}=x)} \dd x \leq 2\sum_{s=1}^{\infty}e^{-s\epsilon^2/(8V)}\leq \frac{2}{e^{\epsilon^2/(8V)}-1}\leq \frac{16V}{\epsilon^2},
\end{align}
where the first inequality is due to $\PP(\theta_{1s}\geq \mu_1-\epsilon\mid \hat{\mu}_{1s}\geq \mu_1-\epsilon)\geq 1/2$ and from Lemma \ref{lem:maximal-inequality}, $\PP( \hat{\mu}_{1s}\leq \mu_1-\epsilon/2)\leq e^{-s\epsilon^2/(8V)}$, and the last inequality is due to $e^{x}-1\geq x$ for all $x>0$. \\
\textbf{Bounding term $A_3$.}
Note that the closed form of the probability density function of $\hmu_{1s}$ is hard to compute. Nevertheless, we only need to find an upper bound of the integration in $A_3$. In the following lemma, we show that it is possible to find such an upper bound with an explicit form. 
\begin{lemma}
\label{lem:exptofunction}
Let  $q(x)=|(s\cdot \kl(x,\mu_1))'|e^{-s\cdot \kl(x,\mu_1)}=s\int_{x}^{\mu_1}1/V(t) \dd t \cdot  e^{-s\cdot \kl(x,\mu_1)}$, $g(x)=e^{sb_s\cdot \kl(x,\mu_1-\epsilon)}$ and $p(x)$ be the PDF of distribution of $\hmu_{1s}$, then 
\begin{align*}
  \int_{\mu_1-\epsilon-\alpha_s}^{\mu_1-\epsilon} q(x) g(x) \dd x +e^{-s\cdot \kl(\mu_1-\epsilon-\alpha_s,\mu_1)}\cdot g(\mu_1-\epsilon-\alpha_s)\geq \int_{\mu_1-\epsilon-\alpha_s}^{\mu_1-\epsilon} p(x) g(x) \dd x.
\end{align*}
\end{lemma}
The proof of Lemma \ref{lem:exptofunction} could be found in Section \ref{sec:technical-lem}.
Besides, we need the following inequality on kl-divergence, which resembles the three-point identity property. In particular, for $\mu_1-\epsilon>x$, we have
\begin{align}
\label{eq:kl-underest}
    -\kl(x,\mu_1)+\kl(x,\mu_1-\epsilon)&=-\int_{x}^{\mu_1} \frac{t-x}{V(t)} \dd t +\int_{x}^{\mu_1-\epsilon}  \frac{t-x}{V(t)} \dd t \notag \\
    &=-\int_{\mu_1-\epsilon}^{\mu_1} \frac{t-x}{V(t)} \dd t\notag\\
    &\leq -\int_{\mu_1-\epsilon}^{\mu_1} \frac{t-(\mu_1-\epsilon)}{V(t)} \dd t \notag \\
    &= -\kl(\mu_1-\epsilon,\mu_1),
\end{align}
where the first and the last equality is due to \eqref{eq:pro-eq1-1}.
For term $A_3$, we have
\begin{align}
A_3&\leq \sum_{s=1}^{M}\int_{\mu_1-\epsilon-\alpha_s}^{\mu_1-\epsilon} p(x) e^{sb_s\cdot \kl(x,\mu_1-\epsilon)} \dd x \label{eq:rots-minimax-under-0} \\
&\leq \sum_{s=1}^{M}\int_{\mu_1-\epsilon-\alpha_s}^{\mu_1-\epsilon}\Big[ q(x) \cdot e^{s\cdot \kl(x,\mu_1-\epsilon)} \Big] \dd x  +\sum_{s=1}^{M}e^{-s\cdot \kl(\mu_1-\epsilon-\alpha_s,\mu_1)} \cdot e^{s\cdot \kl(\mu_1-\epsilon-\alpha_s,\mu_1-\epsilon)}\notag \\
   & \leq  \sum_{s=1}^{M} \int_{\mu_1-\epsilon-\alpha_s}^{\mu_1-\epsilon}\Big[  |s\cdot \kl(x,\mu_1)'| \cdot e^{-s\cdot \kl(x,\mu_1)} \cdot e^{s\cdot \kl(x,\mu_1-\epsilon)} \Big] \dd x+e^{-s\epsilon^2/(2V)} \cdot M  \notag \\
     & \leq  \sum_{s=1}^{M} e^{-s\cdot \kl(\mu_1-\epsilon,\mu_1)}\cdot \int_{\mu_1-\epsilon-\alpha_s}^{\mu_1}\big[  |s\cdot \kl(x,\mu_1)|'   \big] \dd x + e^{-s\epsilon^2/(2V)} \cdot M  \notag \\
   & \leq  \sum_{s=1}^{M}e^{-s\epsilon^2/(2V)}(1+ s\cdot\kl(\mu_1-\epsilon-\alpha_s,\mu_1)) \notag \\
   & \leq  \sum_{s=1}^{M}e^{-s\epsilon^2/(2V)}(1+ 4\log(T/s)) \label{eq:rots-minimax-under},
\end{align}
where the first inequality is due to \eqref{eq:perpty1}, the second inequality is due to Lemma \ref{lem:exptofunction} and $b_s\leq 1$, the third inequality is due to \eqref{eq:kl-underest}, the fourth inequality is due to \eqref{eq:kl-underest}, and the last inequality is due to Lemma \ref{lem:book-minimax} and the definition of $\alpha_s$.
Let $d=\lceil V/\epsilon^2 \rceil$.
For term $\sum_{s=1}^{d} \log(T/s)$, we have
\begin{align}
\label{ddl-1}
    \sum_{s=1}^{d} \log(T/s) &=d\log T-\sum_{s=1}^{d} \log s \notag \\
    & \leq d\log T - \bigg((s\log s-s)\bigg|^{d}_1-\log d \bigg) \notag \\
    & \leq d\log (T/d)+d+\log d\notag\\
    &=O\bigg( \frac{V\log (T\epsilon^2/V) }{\epsilon^2} \bigg),
\end{align}
where the first inequality is due to $\sum_{x=a}^b f(x)\geq \int_{a}^{b} f(x) \dd x -\max_{x\in [a,b]} f(x)$ for monotone function $f$.
For term $\sum_{s=d}^{M} e^{-s\epsilon^2/(2V)} \log(T/s)$, we have
\begin{align}
\label{ddl-0}
    \sum_{s=d}^{M} e^{-s\epsilon^2/(2V)} \log(T/s) &\leq \log(T/d) \sum_{s=d}^{M}e^{-s\epsilon^2/2V} \notag \\
    &\leq \log(T/d) \sum_{s=1}^{\infty}e^{-s\epsilon^2/2V} \notag \\
    & \leq  \frac{\log(T/d)}{e^{\epsilon^2/(2V)}-1}\notag\\
    &\leq \frac{2V\log(T/d)}{\epsilon^2}\notag\\
    &=O\bigg( \frac{V\log(T\epsilon^2/V)}{\epsilon^2}\bigg),
\end{align}
where the fourth inequality is due to $e^x\geq 1+x$ for $x>0$. Substituting \eqref{ddl-0} and \eqref{ddl-1} to \eqref{eq:rots-minimax-under}, we have $A_3=O\big(V\log(T\epsilon^2/V)/\epsilon^2\big)$. Substituting the bounds of $A_1$, $A_2$, and $A_3$ to \eqref{eq:finite-underest-0},  we have
\begin{align*}
\sum_{s=1}^M \EE_{\hmu_{1s}}\bigg[\bigg(\frac{1}{G_{1s}(\epsilon)}-1\bigg)\cdot \ind\{\hmu_{1s}\in L_s \} \bigg]=O\bigg(\frac{V\log(T\epsilon^2/V)}{\epsilon^2} \bigg),
\end{align*}
which completes the proof.

\subsubsection{Proof of Lemma \ref{lem:mini-overest-1}}
Let $\cT=\{t\in [K+1, T]: 1-F_{iT_{i}(t-1)}(\mu_1-\epsilon)>V/(T\epsilon^2)\}$. Then, 
\begin{align}
\label{eq:rots-asym-r0}
   &\EE\left[\sum_{t=K+1}^{T} \ind\{A_t=i, E_{i,\epsilon}^c(t) \} \right] \notag \\
   &\leq  \EE\left[ \sum_{t\in \cT}\ind\{A_t=i\}\right]+\EE \left[ \sum_{t\notin \cT} \ind \{E_{i,\epsilon}^c(t)\}\right] \notag \\
   &\leq \EE\left[ \sum_{t\geq K+1}\bigg(\ind\{A_t=i\}\cdot \ind\{1-F_{iT_{i}(t-1)}(\mu_1-\epsilon)>V/(T\epsilon^2)\}\bigg)\right]+\EE \left[ \sum_{t\notin \cT} V/(T\epsilon^2)\right] \notag \\
   &\leq  \EE\left[ \sum_{s\in [T]}\ind\{1-F_{is}(\mu_1-\epsilon)>V/(T\epsilon^2)\}\right]+\frac{V}{\epsilon^2} \notag \\
   &\leq  \EE \left[\sum_{s=1}^{T} \ind\{G_{is}(\epsilon)>V/(T\epsilon^2) \} \right]+ \frac{V}{\epsilon^2}.
\end{align}
Let $s\geq N\log(T\epsilon^2/V)/{\kl(\mu_i+\rho_i,\mu_1-\epsilon)}$ and $X_{is}$ be a sample from the distribution $\cP(\hmu_{is},s)$, if $\hmu_{is}\leq \mu_i+\rho_i$, we have
\begin{align}
\label{eq:rots-asym-}
    \PP(X_{is}\geq \mu_1-\epsilon)&\leq \exp\big(-sb_s\kl(\hmu_{is},\mu_1-\epsilon)\big) \leq \exp \big(-sb_s\kl(\mu_i+\rho_i,\mu_1-\epsilon) \big) \leq \frac{V}{T\epsilon^2},
\end{align}
where the first inequality is from \eqref{eq:perpty1}, the second inequality is due to the assumption $\hmu_{is}\leq \mu_i+\rho_i$, and the last inequality is due to $s\geq  N\log(T\epsilon^2/V)/{\kl(\mu_i+\rho_i,\mu_1-\epsilon)}$ and $b_s=1-1/(s+1)\geq 1-{\kl(\mu_i+\rho_i,\mu_1-\epsilon)}/\log(T\epsilon^2/V)\geq 1/N$. Note that when $\PP(X_{is}\geq \mu_1-\epsilon)\leq V/(T\epsilon^2)$ holds,  $\ind\{G_{is}(\epsilon)>V/(T\epsilon^2)\}=0$. Now, we check the assumption $\hat{\mu}_{is}\leq \mu_i+\rho_i$ that is needed for \eqref{eq:rots-asym-}. From Lemma \ref{lem:maximal-inequality}, we have $\PP(\hmu_{is}>\mu_i+\rho_i)\leq \exp(-s\rho_i^2/(2V))$. Furthermore, it holds that
\begin{align}
    \label{eq:rots-asym-0}
    \sum_{s=1}^{\infty} e^{-\frac{s\rho_i^2}{2V}} \leq \frac{1}{e^{\rho_i^2/(2V)}-1} \leq \frac{2V}{\rho_i^2},
\end{align}
where the last inequality is due to the fact $1+x\leq e^x$ for all $x$. Let $Y_{is}$ be the event that $\hmu_{is}\leq \mu_i+\rho_i$ and $m=N\log(T\epsilon^2/V)/{\kl(\mu_i+\rho_i,\mu_1-\epsilon)}$. We further obtain
\begin{align}
\label{eq:rots-asym-r1}
    \EE \left[\sum_{s=1}^{T} \ind\{G_{is}(\epsilon)>V/(T\epsilon^2) \} \right]&\leq \EE \left[\sum_{s=1}^{T} [\ind\{G_{is}(\epsilon)>V/(T\epsilon^2)  \}\mid Y_{is} ] \right]+\sum_{s=1}^{T}(1-\PP[Y_{is}]) \notag \\
    & \leq  \EE \left[\sum_{s=\lceil m \rceil}^{T} [\ind\{\PP(X_{is}>\mu_1-\epsilon)>V/(T\epsilon^2))  \}\mid Y_{is} ] \right] \notag \\
    & \qquad  +\lceil m \rceil+\sum_{s=1}^{T}(1-\PP[Y_{is}]) \notag \\
    & \leq \lceil m \rceil+\sum_{s=1}^{T}(1-\PP[Y_{is}]) \notag \\ 
    &\leq 1+\frac{2V}{\rho_i^2}+\frac{N\log(T\epsilon^2/V)}{\kl(\mu_i+\rho_i,\mu_1-\epsilon)},
\end{align}
where the first inequality is due to the fact that $\PP(A)\leq \PP(A\mid B)+1-\PP(B)$,  the third inequality is due to \eqref{eq:rots-asym-} and the last inequality is due to   \eqref{eq:rots-asym-0}. Substituting \eqref{eq:rots-asym-r1} into \eqref{eq:rots-asym-r0}, we complete the proof.

\section{Proof of the Asymptotic Optimality of $\algname$}
Now we prove the asymptotic regret bound \eqref{eq:regret-expts-3} of $\algname$ presented in Theorem \ref{thm:regret_rots}.

\subsection{Proof of the Main Result}
The proof in this section shares many components with the finite-time regret analysis presented in Section \ref{sec:proof_finite}. 
We reuse the decomposition \eqref{eq:decomp_num_pull_i} by specifying $\epsilon=1/\log \log T$. In what follows, we bound  terms $A$ and $B$, respectively.

\paragraph{Bounding Term $A$:}
We reuse  Lemma \ref{lem:underest-decomp}. Then, it only remains term $\sum_{s=1}^{M}\EE\big[\big(1/G_{1s}(\epsilon)-1\big)\cdot \ind\{\hmu_{1s}\in L_s\} \big]$ to be bounded. We bound this term by the following lemma.
\begin{lemma}
\label{lem:rots-asym-A}
Let $\epsilon=1/\log \log T$. Let $M$, $G_{1s}(\epsilon)$, and $L_s$ be the same as defined in Lemma \ref{lem:underest-decomp}. Then, 
 $$\sum_{s=1}^{M}  \EE_{\hmu_{1s}} \Bigg[\bigg(\frac{1}{G_{1s}(\epsilon)} -1\bigg) \cdot \ind\{\hmu_{1s}\in L_s \} \Bigg] =O(V^2(\log \log T)^6+V(\log \log T)^2+1).$$
\end{lemma}
Let $\epsilon=1/\log \log T$.  Combining Lemma \ref{lem:rots-asym-A} and Lemma \ref{lem:underest-decomp} together, we have
\begin{align*}
    A= O\big(V^2(\log \log T)^6+V(\log \log T)^2+1\big).
\end{align*}

\paragraph{Bounding Term $B$:}
 Let $\rho_i=\epsilon=1/\log \log T$. Applying Lemma \ref{lem:mini-overest-1}, we have 
\begin{align}
    B&=\EE\left[\sum_{t=K+1}^{T} \ind\{A_t=i, E_{i,\epsilon}^c(t) \} \right]\notag \\
    &=O(1+V(\log \log T)^2)+\frac{N\log(T/(V(\log \log T)^2))}{\kl(\mu_i+1/\log \log T,\mu_1-1/\log \log T)}.
\end{align}

\paragraph{Putting It Together:}
 Substituting the  bound of term $A$ and $B$ into \eqref{eq:decomp_num_pull_i}, we have
\begin{align*}
    \EE[T_{i}(T)]=O(1+V^2(\log \log T)^6+V(\log \log T)^2)+\frac{N\log(T/(V(\log \log T)^2))}{\kl(\mu_i+1/\log \log T,\mu_1-1/\log \log T)}.
\end{align*}
Note that for  $T\rightarrow +\infty$, $N \rightarrow 1$.
Therefore,
\begin{align*}
    \lim_{T\rightarrow +\infty} \frac{ \EE[T_{i}(T)]}{\log T}= \frac{1}{\kl(\mu_i,\mu_1)}.
\end{align*}
This completes the proof of asymptotic regret.

\subsection{Proof of Lemma \ref{lem:rots-asym-A}}
The proof of this part shares many elements with the proof of Lemma \ref{lem:minimax-undetestimation-rots}. The difference starts at bounding term $A_3$. \\
\textbf{Bounding term $A_3$.}
We need to bound the term $\int_{\mu_1-\epsilon-\alpha_s}^{\mu_1-\epsilon} p(x) e^{s\kl(x,\mu_1-\epsilon)}\dd x$.
We divide the interval $[\mu_1-\epsilon-\alpha_s,\mu_1-\epsilon]$ into $n$ sub-intervals $[x_0,x_1),[x_1,x_2),\cdots, [x_{n-1},x_n]$, such that $ x_0\leq x_1\leq \cdots \leq x_n$. For $i\in[n-1]$, we let 
\begin{align}
\label{eq:asym_xi}
 x_i=\sup_{x: x\leq \mu_1-\epsilon} 4\log(T/ e^{i+1})/s <\kl(x,\mu_1)\leq 4\log(T/ e^{i})/s.
\end{align}
Let $n=\lceil \log T \rceil$ and $x_n=\mu_1$. Then, from definition of $\alpha_s$, $\kl(x_0,\mu_1)\geq \kl(\mu_1-\epsilon-\alpha_s,\mu_1)$. Thus, $x_0\leq \mu_1-\epsilon-\alpha_s$. Now, continue on \eqref{eq:rots-minimax-under-0}, we have 
 \begin{align}
 \label{eq:exp-asym-1}
    \int_{\mu_1-\epsilon-\alpha_s}^{\mu_1-\epsilon} p(x) e^{sb_s\cdot \kl(x,\mu_1-\epsilon)} \dd x&\leq \sum_{i=0}^{n}\int_{x_i}^{x_{i+1}}p(x)e^{sb_s\kl(x,\mu_1-\epsilon)} \dd x  \notag\\
    &\leq \sum_{i=0}^{n} e^{sb_s\kl(x_i,\mu_1)} \int_{x_i}^{x_{i+1}}p(x) \dd x  \notag \\
    & \leq \sum_{i=0}^{n} e^{sb_s\kl(x_i,\mu_1)} e^{-s\cdot\kl(x_{i+1},\mu_1)}  \notag\\
    &\leq \sum_{i=0}^{n} \bigg(\frac{T}{e^i}\bigg)^{b_s} \bigg(\frac{e^{i+1}}{T}\bigg) \notag\\ 
    & =O\bigg(\int_{0}^{\ln T} \bigg(\frac{T}{e^{x}}\bigg)^{b_s}\cdot \frac{e^{x+1}}{T}\dd x+e \bigg)\notag\\
    &=O\bigg(\frac{1}{1-b_s}\bigg)\notag\\
    &=O(s),
 \end{align}
where the first inequality is due to $x_0\leq \mu_1-\epsilon-\alpha_s$ and $x_n=\mu_1\geq \mu_1-\epsilon$, the fourth inequality is due to the definition of $x_i$, and the first equality is due to $\sum_{x=a}^b f(x)\leq \int_{a}^{b} f(x) \dd x +\max_{x\in [a,b]} f(x)$ for monotone function $f$. Now, we bound term $A_3$ as follows.
\begin{align}
\label{eq:rots:asym-A_3}
    A_3&\leq \sum_{s=1}^{M}\int_{\mu_1-\epsilon-\alpha_s}^{\mu_1-\epsilon} p(x) e^{sb_s\cdot \kl(x,\mu_1-\epsilon)} \dd x \notag \\
    & \leq \sum_{s=1}^{\lceil 4V(\log \log  T)^3\rceil }\int_{\mu_1-\epsilon-\alpha_s}^{\mu_1-\epsilon} p(x) e^{sb_s\cdot \kl(x,\mu_1-\epsilon)} \dd x \notag \\
    &\qquad +\sum^{M}_{s=\lceil 4V(\log \log T)^3\rceil }\int_{\mu_1-\epsilon-\alpha_s}^{\mu_1-\epsilon} p(x) e^{sb_s\cdot \kl(x,\mu_1-\epsilon)} \dd x \notag \\
    &\leq \underbrace{O\left(\sum_{s=1}^{\lceil 4V(\log \log T)^3\rceil } s  \right)}_{I_1}+\underbrace{\sum^{M}_{s=\lceil 4V(\log \log T)^3\rceil }e^{-s\epsilon^2/(2V)}(1+4\log T)}_{I_2},
\end{align}
where the first inequality is from \eqref{eq:rots-minimax-under-0} and the last inequality is from \eqref{eq:exp-asym-1} and \eqref{eq:rots-minimax-under}. For term $I_1$, we have $I_1=O(V^2(\log \log T)^{6}+1)$. Let $\epsilon=1/\log \log T$, then $M\leq O(V\log T\cdot (\log \log T)^2)$. For $s\geq 4V(\log \log T)^3$, we have 
$e^{-s\epsilon^2/(2V)}=1/\log^2 T$. Thus, $I_2=O(M/\log T)=O(V (\log\log T)^2)$.  Therefore, 
\begin{align}
\label{eq:rots:asym-A_3-1}
    A_3= O(V^2(\log \log T)^6+V(\log \log T)^2+1).
\end{align}
From \eqref{bounding-a_1} and \eqref{boundinga_2}, we have
\begin{align*}
    A_1+A_2=O(V(\log \log T)^2).
\end{align*}
Substituting the bound of $A_1$, $A_2$ and $A_3$ to  \eqref{eq:finite-underest-0}, we have
$$\sum_{s=1}^{M}  \EE_{\hmu_{1s}} \Bigg[\bigg(\frac{1}{G_{1s}(\epsilon)} -1\bigg) \cdot \ind\{\hmu_{1s}\in L_s \} \Bigg] \leq A_1+A_2+A_3=O(V^2(\log \log T)^6+V(\log \log T)^2+1).$$ This completes the proof.

\section{Proof of Theorem \ref{thm:Gaussian} (Gaussian-TS)}\label{sec:proof_Gaussian_TS}
The proof of Theorem \ref{thm:Gaussian} is similar to that of Theorem \ref{thm:regret_rots}. Thus we reuse the notation in the proofs of Theorem \ref{thm:regret_rots} presented in Sections \ref{sec:proof_finite} and \ref{sec:proof_asym}. However, the sampling distribution $\cP$ in Theorem \ref{thm:Gaussian} is chosen as a Gaussian distribution, and therefore, the concentration and anti-concentration inequalities for Gaussian-TS are slightly different from those used in previous sections. This further affects the results of the supporting lemmas whose proofs depend on the concentration bound of $\cP$. In this section, we will prove the regret bounds of Gaussian-TS by showing the existence of  counterparts of these lemmas for Gaussian-TS.

\subsection{Proof of the Finite-Time Regret Bound}
\label{sec:Gau-finite}
From Lemma \ref{lem:maximal-inequality}, the Gaussian posterior $\cN(\mu,V/n)$ satisfies $\PP(\theta\leq \mu-x)\leq e^{-nx^2/(2V)}$. Hence, \ref{lem:underest-decomp} also holds for Gaussian-TS. The proof of Lemma \ref{lem:minimax-undetestimation-rots} needs to call \eqref{eq:perpty1} and \eqref{eq:perpty2}. However, the tail bound for Gaussian distribution has a different form. We need to replace Lemma \ref{lem:minimax-undetestimation-rots} with the following variant.
\begin{lemma}
\label{lem:Gau-finite}
Let $M$, $G_{1s}(\epsilon)$, and $L_s$ be the same as defined in Lemma \ref{lem:underest-decomp}. Then,
\begin{align*}
\sum_{s=1}^{M}\EE_{\hmu_{1s}} \Bigg[\bigg(\frac{1}{G_{1s}(\epsilon)} -1\bigg) \cdot \ind\{\hmu_{1s}\in L_s \} \Bigg]=O\bigg(\frac{V\log(T\epsilon^2/V)}{\epsilon^2} \bigg).
\end{align*}
\end{lemma}
\noindent In Section \ref{sec:proof_finite}, the proof of Lemma \ref{lem:mini-overest-1} only uses the following property of the sampling distribution: let $X_{is}$ be a sample from $\cP(\hmu_{is},s)$ and if $\hmu_{is}\leq \mu_1-\epsilon$, then
\begin{align*}
    \PP(X_{is}\geq \mu_1-\epsilon)\leq \exp(-sb_s\cdot \kl(\hmu_{is},\mu_1-\epsilon)),
\end{align*}
where the $\kl(\cdot)$ function is defined for Gaussian distribution with variance $V$.
For Gaussian  distribution, let $X_{is}$ be a sample from $\cN(\hmu_{is},V/s)$.  Then from Lemma \ref{lem:maximal-inequality}
\begin{align*}
    \PP(X_{is}\geq \mu_1-\epsilon)\leq \exp(-s\cdot \kl(\hmu_{is},\mu_1-\epsilon)) \leq \exp(-sb_s\cdot \kl(\hmu_{is},\mu_1-\epsilon)),
\end{align*}
where the last inequality is due to $b_s\leq 1$. The other parts of the proof of the finite-time bound are the same as that of Theorem \ref{thm:regret_rots} and thus are omitted.

\subsection{Proof of the Asymptotic Regret Bound}
The proof of Lemma \ref{lem:rots-asym-A} needs to call \eqref{eq:perpty1} and \eqref{eq:perpty2}. However, the tail bound for Gaussian distribution has a different form. We need to replace Lemma \ref{lem:minimax-undetestimation-rots} with the following variant. 
\begin{lemma}
\label{lem:Gau-asym}
Let $M$, $G_{1s}(\epsilon)$, and $L_s$ be the same as defined in Lemma \ref{lem:underest-decomp} and  let $\epsilon=1/\log \log T$. Then, 
\begin{align*}
\lim_{T\rightarrow \infty}\sum_{s=1}^{M} \EE_{\hmu_{1s}} \Bigg[\bigg(\frac{1}{G_{1s}(\epsilon)} -1\bigg) \cdot \ind\{\hmu_{1s}\in L_s \} \Bigg]/\log T =0.
\end{align*}
\end{lemma}
The other parts of asymptotic regret bound are the same as that in Theorem \ref{thm:regret_rots} and are omitted.

\subsection{Proof of Supporting Lemmas}

\subsubsection{Proof of Lemma \ref{lem:Gau-finite}} 
Let $Z$ be a sample from $\cN(\hmu_{1s}, V/s)$ and $\hmu_{1s}=\mu_1+x$.
For $x\leq -\epsilon$, applying  Lemma~\ref{lem:gaussian-tail} with
$z=-\sqrt{s/V}(\epsilon+x)>0$  yields:
for $0<z\leq 1$,
\begin{align}
\label{eq:minimax-G1s-1}
   {G_{1s}(\epsilon)}= \PP(Z>\mu_1-\epsilon) 
   &\geq \frac{1}{2\sqrt{2\pi}}\exp\bigg(-\frac{ s(\epsilon+x)^2}{2V} \bigg).
\end{align}
Besides, for $z>1$,
\begin{align}
\label{eq:minimax-G1s-2}
   {G_{1s}(\epsilon)}&\geq \frac{1}{\sqrt{2\pi}}\frac{z}{z^2+1} e^{-\frac{z^2}{2}}\geq \frac{1}{2\sqrt{2\pi}\cdot z}e^{-\frac{z^2}{2}}= \frac{\sqrt{V}}{-2\sqrt{2\pi}\sqrt{ s}(\epsilon+x)}\exp\bigg(-\frac{s(\epsilon+x)^2}{2V} \bigg).
\end{align}
Since $\hmu_{1s}\sim \cN(\mu_1,V/s)$, $x\sim \cN(0,V/s)$. Let $p(x)$ be the PDF of $\cN(0,V/s)$. Note that $G_{1s}(\epsilon)$ is a random variable with respect to $\hmu_{1s}$ and $\hmu_{1s}=\mu_1+x$.   We have
\begin{align}
\label{eq:finite-underest-Gau0}
   \EE_{\hmu_{1s}} \Bigg[\bigg(\frac{1}{G_{1s}(\epsilon)} -1\bigg) \cdot \ind\{\hmu_{1s}\in L_s \} \Bigg]&\leq  \int_{-\epsilon}^{+\infty} \frac{p(x)}{G_{1s}(\epsilon)} \dd x -1+\int_{-\epsilon-\alpha_s}^{-\epsilon} \frac{ p(x)} {G_{1s}(\epsilon)}\dd x\notag \\
    &\leq  1+\int^{-\epsilon}_{-\epsilon-\alpha_s} \frac{p(x)}{G_{1s}(\epsilon)} \dd x \notag \\
    & \leq 1+ \underbrace{\int_{-\epsilon-\alpha_s}^{-\epsilon} p(x)\left(2\sqrt{2\pi} \cdot\exp\bigg(\frac{s(\epsilon+x)^2}{2V}\bigg)\right) \dd x}_{I_1} \notag \\
    & \quad+\underbrace{\int_{-\epsilon-\alpha_s}^{-\epsilon} p(x)\left(2\sqrt{2\pi} \sqrt{s/V}(-\epsilon-x)\cdot\exp\bigg(\frac{s(\epsilon+x)^2}{2V}\bigg)\right) \dd x}_{I_2}.
 \end{align}
The second inequality is due to the fact that  for $\hmu_{1s}\geq \mu_1-\epsilon$, $G_{1s}(\epsilon)=\PP(Z\geq \mu_1-\epsilon)\geq 1/2$.  The last inequality is due to \eqref{eq:minimax-G1s-1} and \eqref{eq:minimax-G1s-2}.  For term  $I_1$, we have
\begin{align}
\label{eq:boundI1}
  I_1 &=  \int_{-\alpha_s-\epsilon}^{-\epsilon} \left(2\sqrt{\frac{s}{V}} \exp\bigg(\frac{-sx
^2}{2V}\bigg)\exp\bigg(\frac{s(\epsilon+x)^2}{2V}\bigg) \right) \dd x \notag \\
&\leq  2\sqrt{\frac{s}{V}}\exp\bigg(\frac{s\epsilon^2}{2V} \bigg)\int_{-\infty}^{-\epsilon} \exp(s\epsilon x/V) \dd x \notag\\
&= \frac{2\sqrt{V}e^{-s\epsilon^2/(2V)}}{\sqrt{s}\epsilon}.
\end{align}
For term $I_2$, we have
 \begin{align}
 \label{eq:boundi2}
I_2 \leq & \int_{-\alpha_s-\epsilon}^{-\epsilon} \left(2s/V (-\epsilon-x)\exp\bigg(\frac{-sx
^2}{2V}\bigg)\exp\bigg(\frac{s(\epsilon+x)^2}{2V}\bigg) \right) \dd x \notag \\
\leq & 2s/V\exp\bigg(\frac{s\epsilon^2}{2V}\bigg)\int_{-\alpha_s-\epsilon}^{-\epsilon} (-\epsilon-x)\exp(s\epsilon x/V)\dd x\notag\\
 \leq & 2s/V\exp\bigg(\frac{-s\epsilon^2}{2V}\bigg)\int_{-\alpha_s-2\epsilon}^{-2\epsilon} -x\exp(s\epsilon x/V)\dd x  \notag \\
 \leq & 2e\cdot \exp\bigg(\frac{-s\epsilon^2}{2V} \bigg)\alpha_s/\epsilon,
\end{align}
where the last inequality is due to $h(x)=-x\exp(s\epsilon x/V)$ on $x<0$ achieve is maximum at $x=-V/(s\epsilon)$.
We further obtain that   
\begin{align}
\label{eq:mini-s1-M-1}
\sum_{s=1}^{M} \mathbb{E}\bigg[\bigg(\frac{1}{G_{1s}(\epsilon)}-1\bigg) \cdot  \ind\{\hmu_{1s}\in L_s\}\bigg] &=  O\bigg( \sum_{s=1}^{M}\alpha_s/\epsilon + \sum_{s=1}^{M}\frac{\sqrt{V}}{\sqrt{s}\epsilon}+M\bigg) \notag \\
& = O\bigg( \sum_{s=1}^{M}\alpha_s/\epsilon + \int_{s=1}^{M}\frac{\sqrt{V}}{\sqrt{s}\epsilon} \dd s+M\bigg) \notag \\
&= O\bigg( \sum_{s=1}^{M}\alpha_s/\epsilon + \frac{\sqrt{VM}}{\epsilon}+M\bigg)\notag \\
& =O\bigg( \sum_{s=1}^{M}\alpha_s/\epsilon +\frac{V\log(T\epsilon^2/V)}{\epsilon^2}\bigg).
\end{align}
Note that 
\begin{align}
\label{eq:Gau-alpha_s}
    \kl(\mu_1-\epsilon-4\sqrt{V\log(T/s)/s},\mu_1)\geq \kl(\mu_1-\epsilon-4\sqrt{V\log(T/s)/s},\mu_1-\epsilon)=8V\log (T/s)/s,
\end{align}
where the equality is due to \eqref{eq:pinsk}. Thus, from the definition of $\alpha_s$ in \eqref{eq:alphas}, we have $\alpha_s\leq 4\sqrt{V\log(T/s)/s}$. For term $\sum_{s=1}^{M} \alpha_s$, we have
\begin{align}
\label{eq:mini-peel-1}
    \sum_{s=1}^{M} \alpha_s/(4\sqrt{V})& \leq  \sum_{s=1}^{M}\frac{\sqrt{\log (T/s)}}{\sqrt{s}} \notag \\
    &\leq \sum_{j=0}^{\lceil\log M-1\rceil-1}\sum_{s=\lceil e^{j} \rceil}^{\lceil e^{j+1} \rceil} \frac{\sqrt{\log(T/ e^{j})}}{\sqrt{s}} \notag \\
    & \leq  \sum_{j=0}^{\lceil \log M -1 \rceil} \sqrt{\log(T/ e^j)} \int_{e^{j}}^{e^{j+1}} \frac{1}{\sqrt{s}} \dd s+\sum_{j=0}^{\lceil\log M-1\rceil}\frac{\sqrt{\log(T/ e^j)}}{e^{j/2}}  \notag \\
    & \leq  2\sum_{j=0}^{\lceil \log M-1 \rceil} e^{(j+1)/2}\cdot {\log(T/ e^j)} \notag \\
    & \leq  2\sqrt{e}\int_{0}^{\log M} (\log T-x) e^{x/2}  \dd x+2\sqrt{eM}\log(T/M) \notag \\
    & =2\sqrt{e}\bigg( 2\log(e^2T)e^{x/2}-2xe^{x/2} \bigg|_{0}^{\log M} \bigg)+2\sqrt{eM}\log(T/M) \notag \\
    & =O\bigg(\frac{\sqrt{V}\log(T\epsilon^2/V)}{\epsilon}\bigg),
\end{align}
where the third and sixth inequality is due to $\sum_{x=a}^b f(x)\leq \int_{a}^{b} f(x) \dd x +\max_{x\in [a,b]} f(x)$ for monotone function $f$.  Substituting \eqref{eq:mini-peel-1} to \eqref{eq:mini-s1-M-1}, we have
\begin{align*}
 \sum_{s=1}^{M} \mathbb{E}\bigg[\bigg(\frac{1}{G_{1s}(\epsilon)}-1\bigg) \cdot & \ind\{\hmu_{1s}\in L_s\}\bigg]= O\bigg(\frac{V\log(T\epsilon^2/V)}{\epsilon^2}\bigg).
\end{align*}
This completes the proof.

\subsubsection{Proof of Lemma \ref{lem:Gau-asym}} 
The proof of this part is similar to the proof of Lemma \ref{lem:Gau-finite}. We reuse the notation defined in the Lemma \ref{lem:Gau-finite}. Recall $Z$ is a sample from $\cN(\hmu_{1s}, V/s)$.  For $\hmu_{1s}=\mu_1-\epsilon/2$,  from \eqref{eq:kl-1}
\begin{align}
\label{eq:Gau-asym-1}
    \PP(Z\leq \mu_1-\epsilon) \leq \exp\big(-s\epsilon^2/(8V) \big).
\end{align}
We have 
\begin{align}
\label{eq:asym-underest-Gau0}
  & \EE_{\hmu_{1s}} \Bigg[\bigg(\frac{1}{G_{1s}(\epsilon)} -1\bigg) \cdot \ind\{\hmu_{1s}\in L_s \} \Bigg]\notag \\
  & \leq  \int_{-\epsilon/2}^{+\infty} \frac{p(x)}{G_{1s}(\epsilon)} \dd x+ \int^{-\epsilon/2}_{-\epsilon} \frac{p(x)}{G_{1s}(\epsilon)} \dd x -1+\int_{-\epsilon-\alpha_s}^{-\epsilon} \frac{ p(x)} {G_{1s}(\epsilon)}\dd x\notag \\
  &\leq  e^{-s\epsilon^2/(8V)}\int_{-\epsilon/2}^{+\infty} p(x) \dd x +2\int_{-\epsilon}^{-\epsilon/2} {p(x)} \dd x+ \int^{-\epsilon}_{-\epsilon-\alpha_s} {\frac{p(x)}{G_{1s}(\epsilon)}} \dd x \notag \\
  & \leq e^{-s\epsilon^2/(8V)}+ e^{-s\epsilon^2/(8V)}+ \underbrace{\int_{-\epsilon-\alpha_s}^{-\epsilon} p(x)\left(2\sqrt{2\pi} \cdot\exp\bigg(\frac{s(\epsilon+x)^2}{2V}\bigg)\right) \dd x}_{I_1} \notag \\
  & \qquad +\underbrace{\int_{-\epsilon-\alpha_s}^{-\epsilon} p(x)\left(2\sqrt{2\pi} \sqrt{s}(-\epsilon-x)\cdot\exp\bigg(\frac{s(\epsilon+x)^2}{2V}\bigg)\right) \dd x}_{I_2},
\end{align}
where the second inequality is due to \eqref{eq:Gau-asym-1} and the fact that for $x\geq -\epsilon$, $G_{1s}(\epsilon)\geq 1/2$,  the third inequality is due to $x\sim \cN(0, V/s)$ and from \eqref{eq:kl-1}, $\PP(x\leq -\epsilon/2)\leq \exp(-s\epsilon^2/(8V)).$
Further, we have $$\sum_{s=1}^{\infty} \exp(-s\epsilon^2/(8V))\leq \frac{1}{e^{\epsilon^2/(8V)}-1}\leq \frac{8V}{\epsilon^2}.$$ By applying \eqref{eq:boundI1} and \eqref{eq:boundi2} to bound term $I_1$ and $I_2$, we obtain
\begin{align*}
    \sum_{s=1}^{M} \EE_{\hmu_{1s}} \Bigg[\bigg(\frac{1}{G_{1s}(\epsilon)} -1\bigg) \cdot \ind\{\hmu_{1s}\in L_s \} \Bigg] &\leq  \frac{8V}{\epsilon^2}+O\bigg( \sum_{s=1}^{M} \frac{e^{-s\epsilon^2/(2V)}}{\sqrt{s\epsilon}}+\sum_{s=1}^{M}\frac{e^{-s\epsilon^2/(2V)} \alpha_s}{\epsilon}\bigg) \notag \\
    &= O\bigg(\frac{V}{\epsilon^2}+2\sqrt{\log T}/\epsilon\sum_{s=1}^{\infty}e^{-s\epsilon^2/(2V)} \bigg) \notag \\
    & =O\bigg(\frac{V}{\epsilon^2}+\frac{V\sqrt{\log T}}{\epsilon^3} \bigg),
\end{align*}
where the first equality is due to \eqref{eq:Gau-alpha_s}. Let $\epsilon=1/\log \log T$, we have 
\begin{align*}
\lim_{T\rightarrow \infty}\sum_{s=1}^{M} \EE_{\hmu_{1s}} \Bigg[\bigg(\frac{1}{G_{1s}(\epsilon)} -1\bigg) \cdot \ind\{\hmu_{1s}\in L_s \} \Bigg]/\log T =0,
\end{align*}
which competes the proof.

\section{Proof of Theorem \ref{thm:Bernoulli} (Bernoulli-TS)}
Similar to the proof strategy used in Section \ref{sec:proof_Gaussian_TS}, we will prove the regret bounds of Bernoulli-TS via providing a counterpart of the supporting lemma used in the proof of Theorem \ref{thm:Bernoulli}  that depends on the concentration bound of the sampling distribution $\cP$.

\subsection{Proof of the Finite-Time Regret Bound}
Due to the same reason shown in Section \ref{sec:Gau-finite}, we only need to replace Lemma \ref{lem:minimax-undetestimation-rots} with the following variant. The rest of the proof remains the same as that of Theorem \ref{thm:regret_rots}. 
\begin{lemma}
\label{lem:Ber-finite}
Let $M$, $G_{1s}(\epsilon)$, and $L_s$ be the same as defined in Lemma \ref{lem:underest-decomp}. Let  $\epsilon=\Delta/4$. It holds that
\begin{align*}
\sum_{s=1}^{M}\EE_{\hmu_{1s}} \Bigg[\bigg(\frac{1}{G_{1s}(\epsilon)}-1 \bigg) \cdot \ind\{\hmu_{1s}\in L_s \} \Bigg]=O\bigg(\frac{\log(T\epsilon^2)}{\epsilon^2} \bigg).
\end{align*}
\end{lemma}
\subsection{Proof of the Asymptotic Regret Bound}
We note that \citet{agrawal2017near} has proved the asymptotic optimality for Beta posteriors under Bernoulli rewards. One can find the details therein and we omit the proofs here.

\subsection{Proof of Lemma \ref{lem:Ber-finite}}
We first define some notations. Let $F_{n,p}^{B}(\cdot)$ denote the CDF, and $f_{n,p}^{B}(\cdot)$ denote the probability mass function of binomial distribution with parameters $n,p$ respectively. We also let $F_{\alpha,\beta}^{beta}(\cdot)$ denote the CDF of the beta distribution with parameters $\alpha,\beta$. The following equality gives the  relationship between $F_{\alpha,\beta}^{beta}(\cdot)$ and $F_{n,p}^{B}(\cdot)$. 
\begin{align}
    \label{eq:ber-fact}
    F_{\alpha,\beta}^{beta}(y) =1-F^{B}_{\alpha+\beta-1,y}(\alpha-1).
\end{align}
Let $y=\mu_1-\epsilon$. Let $j=S_{i}(t)$ and $s=T_{i}(t)$. 
From \eqref{eq:ber-fact}, we have $G_{1s}(\epsilon)=\PP(\theta_{1}(t)>y)=F_{s+1,y}^{B}(j)$. Note that for Bernoulli distribution, we can set $V=1/4$. Besides, \begin{align*}
    \kl(\mu_1-\epsilon-4\sqrt{V\log(T/s)/s}),\mu_1)\geq \kl(\mu_1-\epsilon-4\sqrt{V\log(T/s)/s}),\mu_1-\epsilon)\geq 8V\log (T/s)/s,
\end{align*}
where the inequality is due to \eqref{eq:pinsk}.
Thus, from the definition of $\alpha_s$ in \eqref{eq:alphas}, we have $\alpha_s\leq 2\sqrt{\log(T/s)/s}$. For $j/s\in L_s$, we have $j/s \geq \mu_1-\epsilon-\sqrt{\frac{2\log (T/s)}{s}}$. Hence,
\begin{align*}
    j\geq ys-2\sqrt{s\log (T/s)}.
\end{align*}
Let $\gamma_s= \lceil ys-2\sqrt{s\log (T/s)} \rceil $.
Therefore, 
\begin{align*}
    \EE_{\hmu_{1s}} \Bigg[\bigg(\frac{1}{G_{1s}(\epsilon)} \bigg) \cdot \ind\{\hmu_{1j}\in L_s \} \Bigg]\leq \sum_{j=\gamma_s}^{s} \frac{f_{s,\mu_1}(j)}{F_{s+1,y}^{B}(j)}.
\end{align*}
In the derivation below, we abbreviate $F^{B}_{s+1,y}(j)$ as $F_{s+1,y}(j)$.  \\
\textbf{Case $s<8/\epsilon$}. From Lemma 2.9 of \citet{agrawal2017near}, we have 
\begin{align}
\label{eq:case<8}
    \sum_{j=1}^{s} \frac{f_{s,\mu_1}(j)}{F_{s+1,y}^{B}(j)} \leq \frac{3}{\epsilon}.
\end{align}
\noindent \textbf{Case $s\geq 8/\epsilon$}. We divide the $Sum(\gamma_s,s)=\sum_{j=\gamma_s}^{s}\frac{f_{s,\mu_1}(j)}{F_{s+1,y}^{B}(j)}$  into four partial
sums: $Sum(\gamma_s,\lfloor ys \rfloor )$, $Sum(\lfloor ys \rfloor,\lfloor ys \rfloor )$, $Sum(\lceil ys \rceil,\lfloor \mu_1 s-\frac{\epsilon}{2}s \rfloor )$, and $Sum(\lceil \mu_1 s-\frac{\epsilon}{2}s \rceil,\lfloor s \rfloor )$ and bound them respectively. We need the following bounds on the CDF of Binomial distribution \citep{jevrabek2004dual} [Prop. A.4].  \\
For $j\leq y(s+1)-\sqrt{(s+1)y(1-y)}$,
\begin{align}
    F_{s+1,y}(j)=\Theta\left(\frac{y(s+1-j)}{y(s+1)-j}\binom{s+1}{j} y^j(1-y)^{s+1-j}\right).
\end{align}
For $j\geq y(s+1)-\sqrt{(s+1)y(1-y)}$,
\begin{align*}
    F_{s+1,y}(j)=\Theta(1).
\end{align*}
\noindent \textbf{Bounding $Sum(\gamma_s,\lfloor ys \rfloor )$.}
Let $R=\frac{\mu_1(1-y)}{y(1-\mu_1)}$. Then we have $R>1$. Using the bounds above, we have for any $j$ that
\begin{align*}
    \frac{f_{s,\mu_1}(j)}{F_{s+1,y}(j)}&\leq \Theta \left( \frac{f_{s,\mu_1}(j)}{\frac{y(s+1-j)}{y(s+1)-j}\binom{s+1}{j} y^j(1-y)^{s+1-j}} \right)+\Theta(1) f_{s,\mu_1}(j) \notag \\
    & =\Theta \bigg(\bigg(1-\frac{j}{y(s+1)}\bigg)R^j \frac{(1-\mu_1)^s}{(1-y)^{s+1}} \bigg) \bigg)+\Theta(1) f_{s,\mu_1}(j).
\end{align*}
This applies that for $s\leq \lfloor ys \rfloor$, 
\begin{align}
\label{eq:finite-ber-1}
    \bigg(1-\frac{j}{y(s+1)}\bigg)R^j \frac{(1-\mu_1)^s}{(1-y)^{s+1}} & =  \frac{y(s+1)-j}{y(1-y)(s+1)} R^{j-ys}R^{ys} \frac{(1-\mu_1)^s}{(1-y)^{s}}\notag \\
   & =\frac{e^{-s\cdot \kl(y,\mu_1)}}{y(1-y)(s+1)} (y(s+1)-j) R^{j-ys},
\end{align}
where the last equality is due to the fact for Bernoulli distribution, $\kl(y,\mu_1)=y\log(y/\mu_1)+(1-y)\log((1-y)/(1-\mu_1))$. Next, we prove $(y(s+1)-j)R^{j-ys}\leq \frac{2R}{R-1}+e/\ln 2$. Consider the following two cases. \\ 
\textbf{Case 1: $1/\ln R \leq y$}.  We have
\begin{align*}
    (y(s+1)-j)R^{j-ys}=  (y(s+1)-j)R^{j-y(s+1)} R^{y} \leq yR^{-y}R^{y} \leq 1,
\end{align*}
where the inequality is due to $xR^{-x}$  is monotone increasing on $x\in (0,1/\ln R)$ and $y(s+1)-j\geq y(s+1)-ys=y\geq 1/\ln R$. \\
\textbf{Case 2: $1/\ln R\geq y$.}  We will divide it into the following three intervals of $R$:  \\
For $R\geq e^2$,  we have
\begin{align*}
(y(s+1)-j)R^{j-ys}&=  (y(s+1)-j)R^{j-y(s+1)} R^{y} \notag \\
& \leq \frac{1}{\ln R} R^{-1/\ln R} R^y \notag \\
& \leq \frac{1}{\ln R} R^{-y}R^{y} \notag \\
& \leq \frac{1}{\ln R}\notag\\
&\leq 1,
\end{align*}
where the first inequality is due to $xa^{-x}$ achieve its maximum at $1/\ln a$. \\
For $2<R<e^2$, 
we have
\begin{align*}
    R^{-1/\ln R}/ \ln R\leq  1/(e\ln 2) \Leftrightarrow  -1 \leq \ln (\ln R /(e\ln 2)) \Leftrightarrow R\geq 2.
\end{align*}
Therefore, 
\begin{align*}
    (y(s+1)-j)R^{j-ys}&=  (y(s+1)-j)R^{j-y(s+1)} R^{y} \notag \\
    & \leq \frac{1}{\ln R} R^{-1/\ln R} R \notag \\
    & \leq R/(e\ln 2)\\
    &\leq e/\ln 2.
\end{align*}
For $1<R<2$, we have $\ln R \geq (R-1)-(R-1)^2/2$. Further, 
\begin{align*}
    \frac{R^{-1/\ln R}}{\ln R} \leq \frac{1}{\ln R} \leq \frac{1}{(R-1)-(R-1)^2/2} \leq \frac{1}{(R-1)(1-(R-1)/2)} \leq \frac{2}{R-1}.
\end{align*} 
 We have
\begin{align*}
    (y(s+1)-j)R^{j-ys}&\leq   (y(s+1)-j)R^{j-y(s+1)} R \notag \\
    & \leq \frac{R}{\ln R} R^{-1/\ln R}  \notag \\
    & \leq  \frac{2R}{R-1}.
\end{align*}
Combining \textbf{Case 1} and \textbf{Case 2} together,  we have $(y(s+1)-j)R^{j-ys}\leq \frac{2R}{R-1}+e/\ln 2$. Substituting this into \eqref{eq:finite-ber-1}, we have
\begin{align}
\label{eq:main-bernoulli}
    \bigg(1-\frac{j}{y(s+1)}\bigg)R^j \frac{(1-\mu_1)^s}{(1-y)^{s+1}}&\leq \frac{e^{-s\cdot \kl(y,\mu_1)}}{y(1-y)(s+1)} \bigg(\frac{2R}{R-1}+e/\ln 2 \bigg) \notag \\
   & \leq \frac{2\mu_1e^{-s\cdot \kl(y,\mu_1)}}{y(\mu_1-y)(s+1)}+ \frac{8e^{-s\cdot \kl(y,\mu_1)}}{y(1-y)(s+1)} \notag \\
   & \leq \frac{20e^{-s\cdot \kl(y,\mu_1)}}{\epsilon(s+1)}.
\end{align}
The second inequality is due to  $\frac{R}{R-1}=\frac{\mu_1(1-y)}{\mu_1-y}$. The last inequality is due to
\begin{align*}
    \frac{\mu_1}{y}=\frac{\mu_1}{\mu_1-\epsilon}= \frac{\mu_1}{\mu_1-\Delta_i/4}\leq 4/3<2,  
\end{align*}
and 
\begin{align*}
y(1-y)\geq\Delta_i/4(1-\Delta_i/4)=\epsilon(1-\epsilon)\geq \epsilon/2,
\end{align*}
where the first inequality is because $y(1-y)$ is decreasing for $y\geq 1/2$ and increasing for $y\leq 1/2$ and $y=\mu_1-\Delta_i/4\in [3/(4\Delta_i),1-\Delta_i/4]$, since $\mu_1\in[0,1]$ and $\mu_1\geq\Delta_i$ by definition, the last inequality is due to the fact $\epsilon=\Delta_i/4\leq 1/4$.
Therefore, we have
\begin{align}
\label{eq:sum1}
    Sum(\gamma_s,\lfloor ys \rfloor)&= \sum_{j=\gamma_s}^{\lfloor ys \rfloor} \frac{f_{s,\mu_1}(j)}{F^B_{s+1,y}(j)} \notag \\
   & =\Theta \bigg(\sum_{j=\gamma_s}^{\lfloor ys \rfloor}\bigg(1-\frac{j}{y(s+1)}\bigg)R^j \frac{(1-\mu_1)^s}{(1-y)^{s+1}} \bigg) \bigg)+\Theta(1) \sum_{j=1}^s f_{s,\mu_1}(j)\notag \\ 
    & =  O \bigg( \frac{e^{-s\cdot \kl(y,\mu_1)}(ys-\gamma_s)}{\epsilon(s+1)} \bigg) +\Theta(1) \notag \\
    & = O\bigg(\frac{\sqrt{\log(T/s)/s}}{\epsilon}\bigg)+ \Theta(1) ,
\end{align}
where the second equality is due to \eqref{eq:main-bernoulli}. \\
\textbf{Bounding $Sum(\lfloor ys \rfloor,\lfloor ys \rfloor )$ and $Sum(\lceil ys \rceil,\lfloor \mu_1 s-\frac{\epsilon}{2}s \rfloor )$.} From Lemma 2.9 of \citet{agrawal2017near}, we have 
\begin{align}
\label{eq:sum2}
    Sum(\lfloor ys \rfloor,\lfloor ys \rfloor )\leq 3e^{-s\kl(y,\mu_1)}\leq 3e^{-2s\epsilon^2},
\end{align}
and
\begin{align}
\label{eq:sum3}
    Sum\bigg(\lceil ys \rceil,\bigg\lfloor \mu_1 s-\frac{\epsilon}{2}s \bigg\rfloor \bigg)=\Theta(e^{-s\epsilon^2/2}).
\end{align}
\textbf{Bounding $Sum(\lceil \mu_1s-\frac{\epsilon}{2}s \rceil, s )$.}
For $j\in [\lceil \mu_1s-\frac{\epsilon}{2}s \rceil, s]$, $F_{s+1,y}(j)=\Theta(1)$.  Hence, 
\begin{align}
\label{eq:sum4}
    Sum(\lceil \mu_1s-\frac{\epsilon}{2}s \rceil, s )=\Theta(1). 
\end{align}\\
Combining \eqref{eq:sum1}, \eqref{eq:sum2}, \eqref{eq:sum3} and \eqref{eq:sum4} together, we have that for $s\geq 8/\epsilon$,
\begin{align}
\label{eq:case>8}
     \EE_{\hmu_{1s}} \Bigg[\bigg(\frac{1}{G_{1s}(\epsilon)}-1 \bigg) \cdot \ind\{\hmu_{1j}\in L_s \} \Bigg]\leq \Theta\bigg(1+e^{-s\epsilon^2/2}+e^{-2s\epsilon^2}+\frac{\sqrt{\log(T/s)/s}}{\epsilon}\bigg).
\end{align}
Combining \eqref{eq:case<8} and \eqref{eq:case>8} together, we have
\begin{align*}
   &\sum_{s=1}^{M} \EE_{\hmu_{1s}} \Bigg[\bigg(\frac{1}{G_{1s}(\epsilon)}-1 \bigg) \cdot \ind\{\hmu_{1s}\in L_s \} \Bigg] \notag \\
   &= \sum_{s: 1\leq s <8/\epsilon} \EE_{\hmu_{1s}} \Bigg[\bigg(\frac{1}{G_{1s}(\epsilon)}-1 \bigg) \cdot \ind\{\hmu_{1s}\in L_s \} \Bigg] \notag \\
   & \qquad+\sum_{s:  s \geq 8/\epsilon} \EE_{\hmu_{1s}} \Bigg[\bigg(\frac{1}{G_{1s}(\epsilon)}-1 \bigg) \cdot \ind\{\hmu_{1s}\in L_s \} \Bigg] \notag \\
   &\leq   \Theta\bigg(\frac{1}{\epsilon^2} \bigg)+\Theta \bigg(M+\sum_{s=1}^{\infty}e^{-2s\epsilon^2}+\sum_{s=1}^{\infty}e^{-s\epsilon^2/2}+\sum_{s=1}^{M} \frac{\sqrt{\log(T/s)/s}}{\epsilon} \bigg) \notag \\
   &\leq   \Theta \bigg(\frac{\log(T\epsilon^2)}{\epsilon^2} \bigg),
\end{align*}
where the last inequality is due to the fact $\sum_{s=1}^{\infty}e^{-2s\epsilon^2}\leq 1/(e^{2\epsilon^2}-1) \leq \frac{1}{2\epsilon^2}$ and $\sum_{s=1}^{M}\sqrt{\frac{\log (T/s)}{s}} \leq \Theta (\epsilon^{-1}\log(T\epsilon^2))$ from \eqref{eq:mini-peel-1}.

\section{Proof of the Minimax Optimality of $\algname^+$}
In this section, we prove the worst case regret bound of $\algname^+$ presented in Theorem \ref{thm:expTS_plus_minimax}.

\subsection{Proof of the Main Result}
\paragraph{Regret Decomposition:}
For simplicity, we reuse the notations in Section \ref{sec:proof_finite}.
Let $S_j=\{i\in[K] \mid 2^{-(j+1)} \leq \Delta_i<2^{-j}\}$ be the set of arms whose gaps from the optimal arm are bounded in the interval $[2^{-(j+1)}, 2^{-j})$. 
Define $\gamma=1/2\log_{2}(T/(VK))-3$. Then we know that for any arm $i\in[K]$ that $\Delta_i>4\sqrt{VK/T}=2^{-(\gamma+1)}$, there must exist some $j\leq\gamma$ such that $i\in S_j$. Therefore, the regret of $\algname^+$ can be decomposed as follows.
\begin{align}
    R_{\mu}(T)&= \sum_{i:\Delta_i>0} \Delta_i \cdot \EE[T_{i}(T)] \notag \\
    & \leq\sum_{i:\Delta_i>4\sqrt{VK/T}}\Delta_i \cdot \EE[T_{i}(T)]+\max_{i:\Delta_i<4\sqrt{VK/T}}\Delta_i\cdot T \label{eq:++decomp_num_pull_sj} \\
    &< \sum_{j<\gamma}\sum_{i\in S_{j}} 2^{-j} \cdot  \EE[T_{i}(T)]+4\sqrt{VKT} \label{eq:++decomp_num_pull_sj+},
\end{align}
where in the first inequality we used the fact that $\sum_{i}\EE[T_i(T)]=T$, and in the last inequality we used the fact that $\Delta_i<2^{-j}$ for $\Delta_i\in S_j$.  The expected number of times that Algorithm~\ref{alg:rots} plays arms in set $S_j$ with $j<\gamma$  is bounded as follows.
\begin{align}\label{eq:+decomp_num_pull_sj}
    \sum_{i\in S_{j}}\mathbb{E}[T_i(T)]
     & =  |S_j|+ \sum_{i\in S_{j}}\EE\left[\sum_{t=K+1}^T \ind \{A_t=i, E_{i,\epsilon_j} (t) \}+ \sum_{t=K+1}^T \ind \{A_t=i, E_{i,\epsilon_j}^c (t) \} \right] \notag \\
    & =  |S_j|
   +\underbrace{\sum_{i\in S_{j}} \EE\left[ \sum_{t=K+1}^T \ind \{A_t=i, E_{i,\epsilon_j} (t) \} \right]}_{A}   + \underbrace{\sum_{i\in S_{j}}\EE\left[ \sum_{t=K+1}^T \ind \{A_t=i, E_{i,\epsilon_j}^c (t) \} \right]}_{B},
\end{align}
where  $\epsilon_j>\sqrt{8VK/T}$ is an arbitrary constant.

\paragraph{Bounding Term $A$:}
Define
\begin{align}
\label{eq:alphas+}
    \alpha_s=\sup_{x\in [0,\mu_1-\epsilon-R_{\min})} \kl(\mu_1-\epsilon-x,\mu_1)\leq 4\log^+(T/(Ks))/s, 
\end{align}
where $\log^+(x)=\max\{0, \log x\}$. We decompose the term $\sum_{i\in S_j}\EE\left[ \sum_{t=K+1}^T \ind \{A_t=i, E_{i,\epsilon} (t) \} \right]$ by the following lemma.
\begin{lemma}
\label{lem:thom-bound}
Let $\epsilon_j=2^{-j-2}$.
Let  $M_j=\lceil 16V\log(T\epsilon_j^2/(KV))/\epsilon_j^2 \rceil$.
\begin{align*}
   \sum_{i\in S_j}\EE\left[ \sum_{t=K+1}^T \ind \{A_t=i, E_{i,\epsilon_j} (t) \} \right]& \leq \sum_{s=1}^{M_j}\mathbb{E}\Bigg[\left(\frac{1}{G_{1s}(\epsilon_j)}-1\right) \cdot \ind\{\hmu_{1s}\in L_s \} \Bigg]+\Theta\bigg(\frac{VK}{\epsilon_j^2}\bigg), \notag \\
\end{align*}
where $G_{is}(\epsilon)=1-F_{is}(\mu_1-\epsilon)$, $F_{is}$ is the CDF of $\cP(\hmu_{is},s)$, 
 and    $L_s=\Big(\mu_1-\epsilon-\alpha_s, R_{\max} \Big]$.  \end{lemma}
Now, we bound the remaining term in Lemma \ref{lem:thom-bound}.
\begin{lemma}\label{lem:+minimax-undetestimation-rots}
 Let $M_j$, $G_{1s}(\epsilon_j)$, and $L_s$ be the same as defined in Lemma \ref{lem:thom-bound}. It holds that
\begin{align*}
    \sum_{s=1}^{M_j}\mathbb{E}_{\hmu_{1s}}\left[\bigg(\frac{1}{G_{1s}(\epsilon_j)}-1\bigg) \cdot \ind\{\hmu_{1s}\in L_s\}\right] =O\bigg(\frac{VK\log(T\epsilon_j^2/(KV))}{\epsilon_j^2}\bigg).
\end{align*}
\end{lemma}
Combining Lemma \ref{lem:thom-bound} and Lemma \ref{lem:+minimax-undetestimation-rots} together, we have
\begin{align*}
  A=O\bigg(\frac{VK\log(T\epsilon_j^2/(KV))}{\epsilon_j^2} \bigg).
\end{align*}
\paragraph{Bounding Term $B$:} We have the following lemma that bounds the second term in \eqref{eq:+decomp_num_pull_sj}.
\begin{lemma}
\label{lem:+mini-overest-1}
Let $N_i=\min\{1/(1-(\kl(\mu_i+\rho_i,\mu_1-\epsilon_j))/\log(T\epsilon_j^2/V)),2\}$. For any $\rho_i, \epsilon_j>0$ that satisfies $\epsilon_j+\rho_i<\Delta_i$,  then
\begin{align*}
 \EE\left[\sum_{t=K+1}^{T} \ind\{A_t=i, E_{i,\epsilon_j}^c(t) \} \right] \leq 1+\frac{2V}{\rho_i^2}+\frac{V}{\epsilon_j^2}+\frac{N_i\log(T\epsilon_j^2/(VK))}{\kl(\mu_i+\rho_i,\mu_1-\epsilon_j)}.
\end{align*}
\end{lemma}

\paragraph{Putting it Together:}
Let $\rho_i=\epsilon_j$. Substituting Lemma \ref{lem:thom-bound} and Lemma \ref{lem:+mini-overest-1}  to the regret decomposition \eqref{eq:++decomp_num_pull_sj+}, we obtain
\begin{align*}
     R_{\mu}(T)&\leq  \sum_{j<\gamma}\sum_{i\in S_{j}} \epsilon_j \cdot  \EE[T_{i}(T)]+4\sqrt{VKT} \notag \\
     & = O\left( \sum_{j<\gamma}\frac{KV\log(T\epsilon_j^2/(VK))}{\epsilon_j}+\sqrt{VKT}+\sum_{i\geq 2} \Delta_i\right) \notag \\
     & = O\bigg(8\sqrt{VKT} \cdot \sum_{n=0}^{\infty}\frac{\log 64+ n \log 2}{2^n} +\sum_{i\geq 2} \Delta_i  \bigg) \notag \\
     & =O\bigg(\sqrt{VKT}+\sum_{i\geq 2} \Delta_i\bigg),
\end{align*}
which completes the proof of the minimax optimality.

\subsection{Proof of Supporting Lemmas}
\subsubsection{Proof of Lemma \ref{lem:thom-bound}}
The proof of this lemma shares many element with that of Lemma \ref{lem:underest-decomp}. Let $\cF_{t}=\sigma(A_1,r_1,\cdots,A_{t},r_t)$ be the filtration.  By the definition of $G_{is}(x)$, it holds that
\begin{align}
\label{eq:mini-decom-1}
    G_{1T_{1}(t-1)}(\epsilon_j)=\PP(\theta_{1}(t)\geq \mu_1-\epsilon_j \mid \cF_{t-1}).
\end{align}
 Define $\cE$ to be the event such that  $\hmu_{1s}\in L_s$ holds for all $s\in [T]$. The indicator function can be decomposed based on $\cE$.
\begin{align}
\label{eq:+mini-decom-f1}
     &\sum_{i\in S_j}\EE\left[ \sum_{t=K+1}^T \ind \{A_t=i, E_{i,\epsilon_j} (t)\} \right] \notag\\
     & \leq T\cdot \PP(\cE^c)+  \sum_{i\in S_j}\EE\left[ \sum_{t=K+1}^T \big[\ind \{A_t=i, E_{i,\epsilon_j} (t) \}\cdot \ind\{\hmu_{1T_{1}(t-1)}\in L_{T_{i}(t-1)}\} \big] \right] \notag \\
     & \leq \Theta\bigg(\frac{VK}{\epsilon_j^2}\bigg)+\sum_{i\in S_j}\EE\left[ \sum_{t=K+1}^T \big[\ind \{A_t=i, E_{i,\epsilon_j} (t) \}\cdot \ind\{\hmu_{1T_{1}(t-1)}\in L_{T_{i}(t-1)}\} \big] \right],
\end{align}
where the second inequality is from Lemma \ref{lem:book-minimax} with $b=K$.  Let $A'_{t}=\arg \max_{i\neq 1}\theta_{i}(t)$. Then 
\begin{align}
\label{eq:mini-decom-22}
    \PP(A_{t}=1\mid \cF_{t-1}) & \geq \PP(\{\theta_{1}(t)\geq \mu_1-\epsilon_j\}\cap\big\{\exists i \in S_j: A'_{t}=i, E_{i,\epsilon_j} (t) \big\}\mid \cF_{t-1}) \notag \\
    &= \PP(\theta_{1}(t)\geq \mu_1-\epsilon_j \mid \cF_{t-1})\PP\bigg(\bigcup_{i\in S_j}\{A'_{t}=i, E_{i,\epsilon_j} (t)\}\bigg)\notag\\
    & = \PP(\theta_{1}(t)\geq \mu_1-\epsilon_j \mid \cF_{t-1}) \cdot \sum_{i\in S_j} \PP(A'_{t}=i,E_{i,\epsilon_j} (t)\mid \cF_{t-1}) \notag \\
    & \geq \frac{G_{1T_{1}(t-1)}}{1-G_{1T_{1}(t-1)}} \cdot \sum_{i\in S_j} \PP(A_{t}=i,E_{i,\epsilon_j} (t)\mid \cF_{t-1}),
\end{align}
where the first inequality is due to the fact when both event $\{\theta_{1}(t)\geq \mu_1-\epsilon\}$ and event $\{\exists i \in S_j: A'_{t}=i,E_{i,\epsilon_j} (t) \}$ hold, we must have $\{A_t=1\}$, the first equality is due to $\theta_1(t)$ is conditionally independent of $A_t'$ and $E_{i,\epsilon_j} (t)$ given $\cF_{t-1}$, the second equality is due to that these events are mutually exclusive,  and  the last inequality is from \eqref{eq:mini-decom-1} and the fact that 
\begin{align*}
    \sum_{i\in S_j} \PP(A_{t}=i,E_{i,\epsilon_j} (t)\mid \cF_{t-1}) \leq (1-\PP(\theta_{1}(t)\geq \mu_1-\epsilon_j \mid \cF_{t-1}))\cdot \sum_{i\in S_j} \PP(A_{t}'=i,E_{i,\epsilon_j} (t)\mid \cF_{t-1}),
\end{align*}
which is due to $\{\exists i \in S_{j}: A_t=i, E_{i,\epsilon_j} (t) \ \text{occurs}\}\subseteq \{\exists i \in S_{j}: A'_{t}=i, E_{i,\epsilon_j} (t) \ \text{occurs} \}\cap \{\theta_{1}(t)\leq \mu_1-\epsilon_j\}$ and the two intersected events are conditionally independent given $\cF_{t-1}$. \\
Consider two cases. \textbf{Case 1: $t:T_1(t-1)\leq M_j$.}
We have
\begin{align}
    \label{eq:decom-T1t<M}
   \EE \left[ \sum_{t:T_{1}(t-1)\leq M_j} \sum_{i\in S_j} \PP (A_t=i,E_{i,\epsilon_j}(t))\right] & \leq \EE \left[ \sum_{t:T_{1}(t-1)\leq M_j}\bigg(\frac{1}{G_{1T_{1}(t-1)}(\epsilon_j)}-1 \bigg)\PP(A_t=1 \mid \cF_{t-1})\right]\notag \\
    & \leq \EE \left[ \sum_{t:T_{1}(t-1)\leq M_j}\bigg(\frac{1}{G_{1T_{1}(t-1)}(\epsilon_j)}-1 \bigg)\ind\{ A_t=1\}\right] \notag \\
    &\leq \EE \left[\sum_{s=1}^{M_j} \bigg(\frac{1}{G_{1s}(\epsilon_j)}-1 \bigg) \right].
\end{align}
where the first inequality is from \eqref{eq:mini-decom-22}. \\
 \textbf{Case 2: $t:T \geq T_1(t-1)> M_j$.}  For this case, we have
 \begin{align}
 \label{eq:decom-T1t>M}
 &\EE \left[\sum_{t:T_{1}(t-1)> M_j}^T \ind \{A_t=i,E_{i,\epsilon_j}(t) \}\right] \notag\\
 &\leq \EE \left[\sum_{t:T_{1}(t-1)> M_j}^T \ind \{\theta_{1}(t)< \mu_1-\epsilon_j\}\right] \notag \\
 & \leq T\cdot \PP \big(\exists s> M_j:\hmu_{1s}<\mu_1-\epsilon_j/2\big)\notag \\
 &\qquad +\EE \left[\sum_{t:T_{1}(t-1)> M_j} \ind \{\theta_{1}(t)< \mu_1-\epsilon_j \mid \hmu_{1T_{1}(t-1)}\geq \mu_1-\epsilon_j/2\}\right] \notag \\
 & \leq T\cdot e^{-M_j(\mu_1-(\mu_1-\epsilon_j/2))^2/(2V)}+T\cdot e^{-M_jb_{M_j}\epsilon^2/(8V)}\notag \\
 & \leq \frac{VK}{\epsilon_j^2},
 \end{align}
 where the first inequality is due to the fact that $\{A_t=i,E_{i,\epsilon_j}(t)\}\subseteq \{ \theta_{1}(t)<\mu_1-\epsilon_j\}$,  the third inequality is due to Lemma \ref{lem:maximal-inequality} and \eqref{eq:perpty1}, and the last inequality is due to the fact $M_j\geq 16VK\log(T\epsilon_j^2/(VK))/\epsilon_j^2$ and $b_{M_j}\geq 1/2$. Combining \eqref{eq:+mini-decom-f1}, \eqref{eq:decom-T1t<M}, and \eqref{eq:decom-T1t>M} together, we complete the proof of this lemma.

\subsubsection{Proof of Lemma \ref{lem:+minimax-undetestimation-rots}}
 Let $p(x)$ be the PDF of $\hmu_{1s}$ and $\theta_{1s}$ be a sample from $\cP(\hmu_{1s},s)$.  We have
\begin{align}
\label{eq:finite-underest-00}
   &\sum_{s=1}^{M_j}\EE_{\hmu_{1s}} \Bigg[\bigg(\frac{1}{G_{1s}(\epsilon)} -1\bigg) \cdot \ind\{\hmu_{1s}\in L_s \} \Bigg]\notag\\
   &\leq  \underbrace{\sum_{s=1}^{M_j}\bigg(\int_{\mu_1-\epsilon_j/2}^{R_{\max}} p(x)/  \PP(\theta_{1s}\geq \mu_1-\epsilon_j \mid \hmu_{1s}=x)\dd x-1}_{A_1}\bigg)\notag \\
   &\qquad +\underbrace{\sum_{s=1}^{M_j}\bigg(\int^{\mu_1-\epsilon_j/2}_{\mu_1-\epsilon_j} p(x)/  \PP(\theta_{1s}\geq \mu_1-\epsilon_j \mid \hmu_{1s}=x)\dd x}_{A_2}\bigg)\notag\\
   &\qquad +\underbrace{\sum_{s=1}^{M_j}\int_{\mu_1-\epsilon_j-\alpha_s}^{\mu_1-\epsilon_j}\Big[ p(x)/  \PP(\theta_{1s}\geq \mu_1-\epsilon_j \mid \hmu_{1s}=x) \Big] \dd x}_{A_2},
\end{align}
where the inequality is due to the definition of $L_s$.  \\
\textbf{Bounding term $A_1$.}
Similar to the bounding term $A_1$ in Lemma \ref{lem:minimax-undetestimation-rots}, we divide $\sum_{s=1}^{M_j}$ into two term, i.e., $\sum_{s=1}^{\lfloor 32V/\epsilon_j^2 \rfloor}$ and $\sum^{M_j}_{s=\lceil 32V/\epsilon_j^2 \rceil}$.  We have
\begin{align}
\label{bounding-a_11}
  A_1 &= \sum_{s=1}^{M_j}\bigg(\int_{\mu_1-\epsilon_j/2}^{R_{\max}} \frac{p(x)}  {\PP(\theta_{1s}\geq \mu_1-\epsilon_j \mid \hmu_{1s}=x)} \dd x -1 \bigg)\notag \\
  &\leq \frac{32V}{\epsilon_j^2}+\sum_{s=\lceil 32V/\epsilon_j^2 \rceil}^{M_j}\bigg(\int_{\mu_1-\epsilon_j/2}^{R_{\max}} \frac{p(x)}  {\PP(\theta_{1s}\geq \mu_1-\epsilon_j \mid \hmu_{1s}=x)} \dd x -1\bigg) \notag \\
   &\leq \frac{32V}{\epsilon_j^2} +\sum_{s=\lceil 32V/\epsilon_j^2 \rceil }^{M_j} \bigg(\frac{1}{1-e^{-s/2\cdot\kl(\mu_1-\epsilon_j/2,\mu_1-\epsilon_j)}}-1\bigg)\notag \\
   &\leq \frac{32V}{\epsilon_j^2}+\sum_{s=\lceil 32V/\epsilon_j^2\rceil }^{M_j} \bigg(\frac{1}{1-e^{-s\epsilon_j^2/(16V)}}-1\bigg) \notag \\
   & =\frac{16V}{\epsilon_j^2}+ \sum_{s=\lceil {32V}/{\epsilon_j^2}\rceil }^{M_j} \frac{1}{e^{s\epsilon_j^2/(16V)}-1}\notag\\
   &\leq \frac{32V}{\epsilon_j^2},
\end{align}
For the first inequality, we use the fact that with probability at least $1-1/K\geq 1/2$, $\theta_{1s}=\hat{\mu}_{1s}\geq \mu_1-\epsilon$. For second inequality we use the fact that for $\theta_{1s}=\hat{\mu}_{1s}$, $\theta_{1s}\geq \mu_1-\epsilon$; for $\theta_{1s}\sim \cP$, from \eqref{eq:perpty2}, $\PP(\theta_{1s}\geq \mu_1-\epsilon \mid \hmu_{1s}=x)\geq 1-e^{-sb_s\epsilon^2/(16V)}$. The third inequality is due to \eqref{eq:pinsk}. \\
\textbf{Bounding term $A_2$.} This part is the same as the bounding term $A_2$ in Lemma \ref{lem:minimax-undetestimation-rots}, thus we omit the details. \\
\textbf{Bounding term $A_3$.} The proof of bounding term  $A_3$ is  similar to that of proofs in Lemma \ref{lem:minimax-undetestimation-rots}. We also omit
the details. The results are as follows. 
\begin{align}
   A_3 &\leq K\sum_{s=1}^{M_j}\int_{\mu_1-\epsilon_j-\alpha_s}^{\mu_1-\epsilon_j} p(x) e^{sb_s\cdot \kl(x,\mu_1-\epsilon_j)} \dd x \label{eq:exp+A_3} \\
   &\leq  K\sum_{s=1}^{M_j}e^{-s\epsilon_j^2/(2V)}\Big(1+ 4\log(T/(Ks))\Big) \\
   &=O\bigg(\frac{VK\log(T\epsilon_j^2/(KV))}{\epsilon_j^2}\bigg),
\end{align}
where the first inequality is due to the fact that with probability $1/K$, we sample from $\cP$.
Substituting the bound of $A_1$, $A_1$ and $A_3$ to  \eqref{eq:finite-underest-00}, we have
\begin{align}
\sum_{s=1}^{M_j} \EE_{\hmu_{1s}}\bigg[\bigg(\frac{1}{G_{1s}(\epsilon_j)}-1\bigg)\cdot \ind\{\hmu_{1s}\in L_s \} \bigg]=O\bigg(\frac{VK\log(T\epsilon_j^2/(VK))}{\epsilon_j^2} \bigg) \label{eq:rots-minimax-under-},
\end{align}
which completes the proof.

\subsubsection{Proof of Lemma \ref{lem:+mini-overest-1}}
Similar to \eqref{eq:rots-asym-r0}, we can obtain
\begin{align}
\label{eq:+rots-asym-r0}
   &\EE\left[\sum_{t=K+1}^{T} \ind\{A_t=i, E_{i,\epsilon_j}^c(t) \} \right] 
   \leq  \EE \left[\sum_{s=1}^{T} \ind\{G_{is}(\epsilon_j)>V/(T\epsilon_j^2) \} \right]+ \frac{V}{\epsilon_j^2}.
\end{align}
Let $s\geq N_i\log(T\epsilon_j^2/(VK))/{\kl(\mu_i+\rho_i,\mu_1-\epsilon_j)}$ and $X_{is}$ be a sample from the distribution $\cP^+(\hmu_{is},s)$. Assume $\hmu_{is}\leq \mu_i+\rho_i$. Then from definition of $\cP^+(\hmu_{is},s)$, with probability $1-1/K$, $\hmu_{is}\leq \mu_i+\rho_i$; with probability $1/K$, $X_{is}$ is a random sample from $\cP(\hmu_{is},s)$. Therefore if $\hmu_{is}\leq \mu_i+\rho_i$ and $s\geq N_i$,  we have
\begin{align}
\label{eq:+rots-asym-}
    \PP(X_{is}\geq \mu_1-\epsilon_j)&\leq \exp(-sb_s\kl(\hmu_{is},\mu_1-\epsilon_j))/K \notag\\
    &\leq \exp (-sb_s\kl(\mu_i+\rho_i,\mu_1-\epsilon_j) )/K \notag \\
    & \leq \frac{V}{T\epsilon_j^2},
\end{align}
where the first inequality is from \eqref{eq:perpty1} and the definition of $\cP^+(\mu,n)$, the second inequality is due to the assumption $\hmu_{is}\leq \mu_i+\rho_i$, and  the last inequality is due to $s\geq N_i\log(T\epsilon_j^2/(VK))/{\kl(\mu_i+\rho_i,\mu_1-\epsilon_j)}$ and $b_s\geq 1/N_i$. The rest of proofs are similar to the proofs in Theorem \ref{thm:regret_rots}. Note that when $\PP(X_{is}\geq \mu_1-\epsilon_j)\leq V/(T\epsilon_j^2)$ holds, term $\ind\{G_{is}(\epsilon_j)>V/(T\epsilon_j^2)\}=0$. Now, we check the assumption $\hat{\mu}_{is}\leq \mu_i+\rho_i$ that is needed for \eqref{eq:+rots-asym-}.
 From Lemma \ref{lem:maximal-inequality}, we have $\PP(\hmu_{is}>\mu_i+\rho_i)\leq \exp(-s\rho_i^2/(2V))$. Furthermore, it holds that
\begin{align}
    \label{eq:+rots-asym-0}
    \sum_{s=1}^{\infty} e^{-\frac{s\rho_i^2}{2V}} \leq \frac{1}{e^{\rho_i^2/(2V)}-1} \leq \frac{2V}{\rho_i^2},
\end{align}
where the last inequality is due to the fact $1+x\leq e^x$ for all $x$. Let $Y_{is}$ be the event that $\hmu_{is}\leq \mu_i+\rho_i$ and $m=N_i\log(T\epsilon_j^2/(VK))/{\kl(\mu_i+\rho_i,\mu_1-\epsilon_j)}$. We further obtain
\begin{align}
\label{eq:+rots-asym-r1}
    \EE \left[\sum_{s=1}^{T} \ind\{G_{is}(\epsilon)>V/(T\epsilon_j^2) \} \right]&\leq \EE \left[\sum_{s=1}^{T} [\ind\{G_{is}(\epsilon)>V/(T\epsilon_j^2)  \}\mid Y_{is} ] \right]+\sum_{s=1}^{T}(1-\PP[Y_{is}]) \notag \\
    & \leq  \EE \left[\sum_{s=\lceil m \rceil}^{T} [\ind\{\PP(X_{is}>\mu_1-\epsilon_j)>V/(T\epsilon_j^2))  \}\mid Y_{is} ] \right] \notag \\
    & \qquad  +\lceil m \rceil+\sum_{s=1}^{T}(1-\PP[Y_{is}]) \notag \\
    & \leq \lceil m \rceil+\sum_{s=1}^{T}(1-\PP[Y_{is}]) \notag \\ 
    &\leq 1+\frac{2V}{\rho_i^2}+\frac{N_i\log(T\epsilon_j^2/V)}{\kl(\mu_i+\rho_i,\mu_1-\epsilon_j)},
\end{align}
where the first inequality is due to the fact that $\PP(A)\leq \PP(A\mid B)+1-\PP(B)$,  the third inequality is due to \eqref{eq:+rots-asym-} and the last inequality is due to   \eqref{eq:+rots-asym-0}. Substituting \eqref{eq:+rots-asym-r1} into \eqref{eq:+rots-asym-r0}, we complete the proof.

\section{Proof of the Asymptotic Optimality of $\algname^+$}\label{sec:proof_asym}
Now we prove the asymptotic regret bound of $\algname^+$ presented in Theorem \ref{thm:expTS_plus_minimax}. 

\subsection{Proof of the Main Result}
The proof of the this part shares many elements with finite time regret analysis. In what follows, we bound terms $A$ and $B$, respectively. 
\paragraph{Bounding Term $A$:}
We reuse the Lemma \ref{lem:thom-bound}. Then, it only remains term $\sum_{s=1}^{M_j}\EE\big[\big(1/G_{1s}(\epsilon)-1\big)\cdot \ind\{\hmu_{1s}\in L_s\} \big]$ to be bounded. We bound this term by the following lemma.
\begin{lemma}
\label{lem:+rots-asym-A}
 Let $M_j$, $G_{1s}(\epsilon_j)$, and $L_s$ be the same as defined in Lemma \ref{lem:thom-bound}.
 $$\sum_{s=1}^{M_j}  \EE_{\hmu_{1s}} \Bigg[\bigg(\frac{1}{G_{1s}(\epsilon_j)} -1\bigg) \cdot \ind\{\hmu_{1s}\in L_s \} \Bigg] =O(V^2K(\log \log T)^6+V(\log \log T)^2+K).$$
\end{lemma}
Combining Lemma \ref{lem:thom-bound} and Lemma \ref{lem:+rots-asym-A} together and let $\epsilon=1/\log \log T$, we have
\begin{align*}
  A=O\big((V^2K(\log \log T)^6+V(\log \log T)^2+K\big).
\end{align*}

\paragraph{Bounding Term $B$:}
  Let $\rho_i=\epsilon_j=1/\log \log T$. Applying \eqref{lem:+mini-overest-1}, we have 
\begin{align}
   \EE\left[\sum_{t=K+1}^{T} \ind\{A_t=i, E_{i,\epsilon_j}^c(t) \} \right]=O(V\log^2 \log T)+\frac{N_i\log(T\epsilon_j^2/V)}{\kl(\mu_i+1/\log \log T,\mu_1-1/\log \log T)}.
\end{align}
Therefore, we have
\begin{align*}
    B&=\sum_{i\in S_{j}}\EE\left[ \sum_{t=K+1}^T \ind \{A_t=i, E_{i,\epsilon_j}^c (t) \} \right] \notag \\
    &\leq O(VK\log^2 \log T)+\sum_{i\in S_j}\frac{N_i\log(T/(V\log^2 \log T))}{\kl(\mu_i+1/\log \log T,\mu_1-1/\log \log T)}.
\end{align*}

\paragraph{Putting It Together:}
 Substituting the  bound of term $A$ and $B$ into \eqref{eq:++decomp_num_pull_sj+}, we have
\begin{align*}
    \sum_{i\in S_j}\EE[T_{i}(T)]=O(V^2K\log^6 \log T+VK\log^2 \log T+K)+\sum_{i\in S_j}\frac{N_i\log(T/(V\log^2 \log T))}{\kl(\mu_i+1/\log \log T,\mu_1-1/\log \log T)}.
\end{align*}
Note that for $T\rightarrow \infty$, $N_i\rightarrow 1$. Therefore,
\begin{align*}
    \lim_{T\rightarrow \infty} \sum_{i\in S_j}\frac{ \EE[T_{i}(T)]}{\log T}= \sum_{i\in S_j}\frac{1}{\kl(\mu_i,\mu_1)}.
\end{align*}
This completes the proof of the asymptotic regret.

\subsection{Proof of Lemma \ref{lem:+rots-asym-A}}
The proof of this Lemma shares many elements with the proof of Lemma \ref{lem:+minimax-undetestimation-rots}. We can use the bound of $A_1$ and $A_2$ in Lemma \ref{lem:+minimax-undetestimation-rots}. Let $\epsilon=1/\log \log T$, we have
\begin{align*}
    A_1+A_2=O(V(\log\log T)^2).
\end{align*}For term $A_3$,  from \eqref{eq:exp+A_3}, we have 
\begin{align*}
    A_3 &\leq K\sum_{s=1}^{M_j}\int_{\mu_1-\epsilon_j-\alpha_s}^{\mu_1-\epsilon_j} p(x) e^{sb_s\cdot \kl(x,\mu_1-\epsilon_j)} \dd x.
\end{align*}
Then, similar to \eqref{eq:exp-asym-1}, we have
\begin{align}
 \label{eq:+exp-asym-1}
    K\int_{\mu_1-\epsilon_j-\alpha_s}^{\mu_1-\epsilon_j} p(x) e^{sb_s\cdot \kl(x,\mu_1-\epsilon_j)} \dd x=O(Ks).
 \end{align}
Let $\epsilon=1/\log \log T$. By dividing $\sum_{s=1}^{M_j}$ into two terms $\sum_{s=1}^{\lceil 4V(\log \log T)^3\rceil}$ and $\sum_{s=\lceil 4V(\log \log T)^3\rceil}^{M_j}$, from \eqref{eq:rots:asym-A_3} and \eqref{eq:rots:asym-A_3-1}, we have
\begin{align*}
  A_3=O(V^2K(\log\log T)^6+VK(\log \log T)^2+K).
\end{align*}
Substituting the bound of $A_1$, $A_1$ and $A_3$ to  \eqref{eq:finite-underest-00}, we have
\begin{align*}
    A=O(V^2K(\log\log T)^6+VK(\log \log T)^2+K).
\end{align*}
This completes the proof.

\section{Proof of Technical Lemmas}
\label{sec:technical-lem}
In this section, we present the proofs of the remaining lemmas used in previous sections.

\subsection{Proof of Lemma \ref{lem:book-minimax}}

Let $\kl_{+}(x,y)=\kl(x,y)\ind(x\leq y)$. We only need to prove
$$\PP \bigg(\exists s\leq  f(\epsilon): \kl_+(\hat{\mu}_{1s},\mu_1)\geq 4\log(T/(bs))/s \bigg) =O\bigg(\frac{bV}{T\epsilon^2}\bigg).$$ 
The proof of this step relies on the standard ``peeling technique''. We have
\begin{align}
\label{eq:peeling-1}
    &\PP \bigg(\exists s\leq f(\epsilon): \kl_+(\hat{\mu}_{1s},\mu_1)\geq 4\log(T/(bs))/s \bigg) \notag \\
    &\leq  \sum_{n=0}^{\infty} \PP \bigg(\exists \frac{f(\epsilon)}{2^{n+1}}\leq s\leq \frac{f(\epsilon)}{2^{n}}: \kl_+(\hat{\mu}_{1s},\mu_1)\geq 4\log(T/(bs))/s  \bigg) \notag \\
    &\leq   \sum_{n=0}^{\infty} \PP \bigg(\exists \frac{f(\epsilon)}{2^{n+1}}\leq s\leq \frac{f(\epsilon)}{2^{n}}: \kl_+(\hat{\mu}_{1s},\mu_1)\geq \frac{4\log(T/(b\cdot f(\epsilon)/2^n))}{M/2^{n}}  \bigg) \notag \\
    & \leq \sum_{n=0}^{\infty} \exp \bigg(-\frac{f(\epsilon)}{2^{n+1}} \cdot \frac{4\log(T/(b\cdot f(\epsilon)/2^n))}{f(\epsilon)/2^{n}}  \cdot \bigg) \notag \\
    &= \sum_{n=0}^{\infty}\exp\big(-2\log(T/(b\cdot f(\epsilon)/2^{n}))\big) \notag \\
   &\leq   \sum_{n=0}^{\infty} \bigg(\frac{bf(\epsilon)}{T\cdot 2^{n}}\bigg)^2,
\end{align}
where the third inequality is due to Lemma \ref{lem:maximal-inequality}. Note that $f(\epsilon)\leq 32V\log(T\epsilon^2/(bV))/\epsilon^2$. We have
\begin{align*}
    \frac{bf(\epsilon)}{T\cdot 2^n}\leq  \frac{bf(\epsilon)}{T} \leq 32\log\bigg(\frac{T\epsilon^2}{bV}\bigg)\cdot\frac{bV}{T\epsilon^2}.
\end{align*}\
Continue on equation \eqref{eq:peeling-1}, we have
\begin{align*}
    \sum_{n=0}^{\infty} \bigg(\frac{bf(\epsilon)}{T\cdot 2^{n}}\bigg)^2&\leq \sum_{n=0}^{\infty} \bigg(\frac{bf(\epsilon)}{T\cdot 2^{n}}\cdot 32\log\bigg(\frac{T\epsilon^2}{bV}\bigg)\cdot\frac{bV}{T\epsilon^2}\bigg) \notag \\
    &\leq32^2\sum_{n=0}^{\infty}\bigg(\frac{bV}{T\epsilon^2\cdot 2^{n}}\cdot \bigg(\log\bigg(\frac{T\epsilon^2}{bV}\bigg)\bigg)^2\cdot\frac{bV}{T\epsilon^2}\bigg) \notag \\
    & \leq O\bigg(\frac{bV}{T\epsilon^2} \bigg),
\end{align*}
where the last inequality is due to $(\log(x))^2/x\leq 1$ for $x\geq 1$.

\subsection{Proof of Lemma \ref{lem:exptofunction}}

We decompose the proof of Lemma \ref{lem:exptofunction} into two cases:  Case 1: $p(x)$ is in continues form, and Case 2: $p(x)$ is in discrete form. We first focus on the case that $p(x)$ is in continues form.

Divide the interval $[x_{0}, x_n]$ into $n$ sub-intervals $[x_0,x_{1})$,$[x_{1},x_{2}),\cdots, [x_{n-1},x_n]$, such that $\mu_1-\epsilon-\alpha_s= x_0 \leq x_{1} \leq  \cdots \leq x_{n-1}\leq x_n=\mu_1-\epsilon-b_s$, $\int_{x_0}^{x_1}q(x) \dd x =\int_{x_{i-1}}^{x_{i}}q(x) \dd x $ for all $i\in [n]$. 
We now define a new function $p_n(x)$. Assume $p_{n}(x)$ has been defined on $[x_{i},x_n]$. We define $p_{n}(x)$ on $[x_{i-1},x_{i})$ in the following way.   We consider two cases. \\ \textbf{Case 1:} $\int_{x_{i}}^{x_n} p_{n}(x) \dd x +\int_{x_{i-1}}^{x_{i}} p(x) \dd x \geq e^{-s\cdot \kl(x_n, \mu_1)}-e^{-s\cdot \kl(x_{i-1},\mu_1)}$. Then, we define the function  $p_n(x)=p(x)$ for $x\in [x_{i-1},x_{i})$. \\
\textbf{Case 2:} $ \int_{x_{i}}^{x_{n}} p_{n}(x) \dd x +\int_{x_{i-1}}^{x_{i}} p(x) \dd x< e^{-s\cdot \kl(x_n, \mu_1)}-e^{-s\cdot \kl(x_{i-1},\mu_1)}$. Let $\beta=e^{-s\cdot \kl(x_n, \mu_1)}-e^{-s\cdot \kl(x_{i-1},\mu_1)}-\int_{x_{i}}^{x_{n}} p_{n}(x) \dd x -\int_{x_{i-1}}^{x_{i}} p(x) \dd x$.  Then, define $p_n(x)=p(x)+\beta/(x_{i}-x_{i-1})$.  Hence, for case 2, it holds that
\begin{align}
\label{eq:defp'x}
    \int^{x_n}_{x_{i-1}} p_n(x) \dd x=e^{-s\cdot \kl(x_n, \mu_1)}-e^{-s\cdot \kl(x_{i-1},\mu_1)}.
\end{align}
Let $y_n=x_n=\mu_1-\epsilon$.  For all $i\in [n]$, define $y_i=x_i$ if  $\int^{y_n}_{x_i} p_n(x) \dd x=e^{-s\cdot \kl(\mu_1-\epsilon-b_s, \mu_1)}-e^{-s\cdot \kl(x_i,\mu_1)}$. Otherwise, define $y_i$ such that
\begin{align*}
    \int_{y_i}^{y_n} p_n(x) \dd x=e^{-s\cdot \kl(\mu_1-\epsilon-b_s, \mu_1)}-e^{-s\cdot \kl(x_i,\mu_1)}.
\end{align*}
From the definition,  we know
\begin{align}
\label{eq:full-exp-dis-0}
   x_i\leq y_i.
\end{align}
Since $p_n(x)\geq p(x)$ holds for any $x\in [x_0,x_n]$ and $g(x)\geq 0$, we have 
\begin{align}
\label{eq:full-exp-dis}
    \int_{\mu_1-\epsilon-\alpha_s}^{\mu_1-\epsilon-b_s}p(x) g(x) \dd x \leq  \int_{\mu_1-\epsilon-\alpha_s}^{\mu_1-\epsilon-b_s}p_n(x) g(x) \dd x.
\end{align}
Note that $g(x)$ is monotone decreasing 
\begin{align*}
g'(x)&=\big( e^{sb_s\cdot \kl(x,\mu_1-\epsilon)}\big)'=sb_s (\kl(x,\mu_1-\epsilon))' e^{sb_s\cdot \kl(x,\mu_1-\epsilon)} \notag \\
&=sb_s e^{sb_s\cdot \kl(x,\mu_1-\epsilon)} \bigg( \int_{x}^{\mu_1-\epsilon} \frac{t-x}{V(t)} \dd t \bigg)' \notag \\
    & =sb_s e^{sb_s\cdot \kl(x,\mu_1-\epsilon)} \cdot \int_{x}^{\mu_1-\epsilon}  \frac{-1}{V(t)} \dd t \notag \\
    & \leq 0.
    \end{align*}
We have 
\begin{align}
\label{eq:y0-yn}
    \sum_{i=0}^{n-1}\int^{y_{i+1}}_{y_{i}} p_n(x) g(x)\dd x &\leq  \sum_{i=0}^{n-1} g(y_{i})\int_{y_i}^{y_{i+1}} p_n(x)  \dd x \notag \\
    & \leq \sum_{i=0}^{n-1} g(x_{i})\int_{y_i}^{y_{i+1}} p_n(x)  \dd x=\sum_{i=0}^{n-1} g(x_{i})\int_{x_i}^{x_{i+1}} q(x)  \dd x \notag \\
    & \leq \sum_{i=0}^{n-1} \int_{x_i}^{x_{i+1}} q(x) g(x) \dd x + \sum_{i=0}^{n-1} (g(x_{i})-g(x_{i+1}))\int_{x_i}^{x_{i+1}} q(x)  \dd x \notag \\
    & =  \sum_{i=0}^{n-1} \int_{x_i}^{x_{i+1}} q(x) g(x) \dd x + (g(x_0)-g(x_n])\int_{x_0}^{x_1} q(x) \dd x \notag \\
    & \leq \sum_{i=0}^{n-1} \int_{x_i}^{x_{i+1}} q(x) g(x) \dd x + (g(x_0)-g(x_n))/n.
\end{align}
In the first inequality, we use the fact that  $g(x)$ is monotone decreasing. The second inequality is due to \eqref{eq:full-exp-dis-0}. The first equality is from the fact that 
\begin{align*}
    \int_{x_i}^{x_{i+1}} q(x) \dd x = e^{-s\cdot \kl(x,\mu_1)} \bigg|_{x_i}^{x_{i+1}}=e^{-s\cdot \kl(x_{i+1}, \mu_1)}-e^{-s\cdot \kl(x_i,\mu_1)},
\end{align*}
and the definition of $y_i$ such that 
\begin{align*}
    \int_{y_i}^{y_{i+1}} p_n(x)  \dd x=\int_{y_i}^{y_{n}} p_n(x)  \dd x-\int_{y_{i+1}}^{y_{n}} p_n(x)  \dd x= e^{-s\cdot \kl(x_{i+1}, \mu_1)}-e^{-s\cdot \kl(x_i,\mu_1)}=\int_{x_i}^{x_{i+1}} q(x) \dd x.
\end{align*} The third inequality is due to $\sum_{i=0}^{n-1} \int^{x_i}_{x_{i+1}} q(x) g(x) \dd x\geq  \sum_{i=0}^{n-1} g(x_{i})\int^{x_i}_{x_{i+1}} q(x)  \dd x$.  
Now, we focus on bounding term $\int_{x_0}^{y_0} p_{n}(x) g(x) \dd x$. Note that
\begin{align*}
  \int_{x_0}^{y_0} p_{n}(x) g(x) \dd x \leq g(x_0)  \int_{x_0}^{y_0} p_{n}(x) \dd x.
\end{align*}
Hence, we only need to bound $ \int_{x_0}^{y_0} p_{n}(x) \dd x$. Let
\begin{align*}
    n'=\min\big\{j\in\{0,\cdots,n\}: p_{n}(x)=p(x) \text{\ for all \ } x \in [x_0,x_j)  \big\}.
\end{align*}
From the definition,  for $x< x_{n'}$, $p_{n}(x)=p(x)$. Besides, for $x\in [x_{n'},x_{n'+1})$, it must belong to case 2 in the definition of $p_{n}(x)$. Hence, 
\begin{align*}
    \int_{x_{n'}}^{x_n}p_{n}(x) \dd x=e^{-s\kl(x_n,\mu_1)}-e^{-s\kl(x_{n'},\mu_1)}.
\end{align*} 
Therefore, $y_{n'}=x_{n'}$.
Further, from Lemma \ref{lem:maximal-inequality}, we have 
\begin{align}
\label{eq:yn'}
 \int_{x_0}^{y_{n'}} p_n(x) \dd x = \int_{x_0}^{y_{n'}} p(x) \dd x \leq \Pr(\hmu_{1s}\leq y_{n'})\leq e^{-s\cdot \kl(y_{n'},\mu_1)}.    
\end{align}
Now, we have 
\begin{align}
\label{eq:yn-xn}
    \int^{y_0}_{x_0}p_{n}(x) \dd x &=  \int^{y_{n'}}_{x_0}p_{n}(x) \dd x -\int_{y_{0}}^{y_{n'}}p_{n}(x) \dd x \notag \\
    & \leq   e^{-s\kl({y_{n'}},\mu_1)}-\bigg( \int^{x_n}_{y_{0}}p_{n}(x) \dd x-\int^{x_n}_{y_{n'}}p_{n}(x) \dd x \bigg) \notag \\
    & = e^{-s\kl({y_{n'}},\mu_1)}-  \big(e^{-s\kl(\mu_1-\epsilon,\mu_1)}-e^{-s\kl(x_{n},\mu_1)}-e^{-s\kl(\mu_1-\epsilon,\mu_1)}+e^{-s\kl(x_{n'},\mu_1)} \big) \notag \\
    & = e^{-s\kl(x_0,\mu_1)},
\end{align}
where the first inequality is due to \eqref{eq:yn'}, and  the last inequality we use the fact $y_{n'}=x_{n'}$.
Finally, we have
\begin{align*}
     &\int^{\mu_1-\epsilon}_{\mu_1-\epsilon-\alpha_s}  p(x) g(x) \dd x\\
     &\leq \int^{\mu_1-\epsilon}_{\mu_1-\epsilon-\alpha_s} p_n(x) g(x) \dd x \notag \\
     &= \sum_{i=0}^{n-1}\int^{y_{i+1}}_{y_{i}} p_n(x) g(x) \dd x + \int^{y_0}_{x_0} p_{n}(x) g(x) \dd x  \notag \\
     & \leq \sum_{i=0}^{n-1} \int_{x_i}^{x_{i+1}} q(x) g(x) \dd x + (g(x_0)-g(x_n])/n+ g(x_0)\int^{y_0}_{x_0} p_{n}(x) \dd x \notag \\
     & \leq \sum_{i=0}^{n-1} \int^{x_{i+1}}_{x_{i}} q(x) g(x) \dd x + (g(x_0)-g(x_n])/n+ g(\mu_1-\epsilon-\alpha_s) e^{-s\kl(\mu_1-\epsilon-\alpha_s,\mu_1)} \notag \\
     & =\int_{\mu_1-\epsilon-\alpha_s}^{\mu_1-\epsilon-b_s} q(x) g(x) \dd x+g(\mu_1-\epsilon-\alpha_s) e^{-s\kl(\mu_1-\epsilon-\alpha_s,\mu_1)} + (g(x_0)-g(x_n))/n,
\end{align*}
where the second inequality is due to \eqref{eq:y0-yn}, and the third inequality is due to \eqref{eq:yn-xn}.
Note that
\begin{align*}
   \lim_{n\rightarrow \infty} (g(x_0)-g(x_n))/n =0.
\end{align*}
Therefore, it holds that
\begin{align*}
     \int^{\mu_1-\epsilon}_{\mu_1-\epsilon-\alpha_s}  p(x) g(x) \dd x &\leq \int_{\mu_1-\epsilon-\alpha_s}^{\mu_1-\epsilon} q(x) g(x) \dd x+g(\mu_1-\epsilon-\alpha_s) e^{-s\kl(\mu_1-\epsilon-\alpha_s,\mu_1)} \notag\\
     &\qquad+ (g(x_0)-g(x_n))/n,
\end{align*}
which completes the Lemma for continues form $p(x)$.\\  Next, we assume $p(x)$ is in discrete form. Let $z_0,\cdots, z_k$ be all the points such that $p(z_i)>0$ and $z_i\in[\mu_1-\epsilon-\alpha_s,\mu_1-\epsilon]$. We need to prove
\begin{align*}
\int_{\mu_1-\epsilon-\alpha_s}^{\mu_1-\epsilon-b_s}  q(x)g(x) \dd x + e^{-s\cdot \kl(\mu_1-\epsilon-\alpha_s,\mu_1)}\cdot g(\mu_1-\epsilon-\alpha_s)\geq \sum_{i=0}^{k} p(z_i) g(z_i). 
\end{align*}
 We assume $z_0\leq z_1\leq\cdots\leq z_k$. Define $h(z_i)=\int_{z_{i-1}}^{z_{i}} q(x) \dd x$ for $i\in [K]$ and $h(z_0)=\int_{\mu_1-\epsilon-\alpha_s}^{z_0} q(x) \dd x$.
We have 
\begin{align}
\label{eq:discrete0}
    \int_{\mu_1-\epsilon-\alpha_s}^{\mu_1-\epsilon}  q(x)g(x) \dd x -\sum_{i=0}^{k} p(z_i) g(z_i) &\geq \int_{\mu_1-\epsilon-\alpha_s}^{z_k}  q(x)g(x) \dd x -\sum_{i=0}^{k} p(z_i) g(z_i) \notag \\
    & \geq \sum_{i=1}^{k} g(z_{i})\int_{z_{i-1}}^{z_i} q(x) \dd x+ g(z_0)\int_{\mu_1-\epsilon-\alpha_s}^{z_0} q(x)  \dd x \notag\\
    &\qquad-\sum_{i=0}^{k} p(z_i) g(z_i) \notag \\
    & = \sum_{i=0}^{k} (h(z_i)-p(z_i)) g(z_i).
\end{align}
We define $k'=\min \{j\in \{ 0,\cdots, k\}: p(z_i)-h(z_i) \geq 0 \text{ for all } i\leq j\}$.
If such $k'$ does not exist, then  $p(z_i)-h(z_i) < 0 $ always holds. Hence,  $\sum_{i=0}^{k} (h(z_i)-p(z_i)) g(z_i)\geq 0$.  Otherwise, $k'$ exists. From lemma \ref{lem:maximal-inequality}, we have
\begin{align}
\label{eq:discrete1}
    \sum_{i=0}^{k'} p(z_i)  \leq \Pr(\hmu_{1s}\leq z_{k'})\leq e^{-s\cdot \kl(z_{n'},\mu_1)}.  
\end{align}
Continue on \eqref{eq:discrete0}, we have
\begin{align}
\label{eq:discrete3}
    \sum_{i=0}^{k} (h(z_i)-p(z_i)) g(z_i) & \geq -\sum_{i=0}^{k'} (p(z_i)-h(z_i)) g(z_i) \notag \\
    & \geq -g(\mu_1-\epsilon-\alpha_s) \sum_{i=0}^{k'} p(z_i)-h(z_i) \notag \\
    & =-g(\mu_1-\epsilon-\alpha_s) \bigg(\sum_{i=0}^{k'}p(z_i)- \sum_{i=0}^{k'}h(z_i)\bigg) \notag \\
    & \geq -g(\mu_1-\epsilon-\alpha_s) \bigg(e^{-s\kl(z_{k'},\mu_1)}- \int_{\mu_1-\epsilon-\alpha_s}^{z_{k'}} q(x) \dd x\bigg) \notag \\
    & =  -g(\mu_1-\epsilon-\alpha_s) \cdot e^{-s\kl(\mu_1-\epsilon-\alpha_s,\mu_1)},
\end{align}
where the first inequality is from the definition of $k'$, the second inequality is due to that $g(\cdot)$ is monotone decreasing, the first equality is due to $p(z_i)-h(z_i)\geq 0$ for all $i\in \{0,1,\cdots, k' \}$,  the third inequality is due to \eqref{eq:discrete1}, and the last equality is due to 
\begin{align*}
    \int_{\mu_1-\epsilon-\alpha_s}^{z_{k'}} q(x) \dd x=e^{-s\kl(x,\mu_1)} \bigg|_{\mu_1-\epsilon-\alpha_s}^{z_{k'}}= e^{-s\kl(z_{k'},\mu_1)}-e^{-s\kl(\mu_1-\epsilon-\alpha_s,\mu_1)}.
\end{align*}
Combining \eqref{eq:discrete1} and \eqref{eq:discrete3}, we complete the proof.

\section{Useful Inequalities}
\begin{lemma}[Maximal Inequality~\citep{menard2017minimax}]\label{lem:maximal-inequality}
Let $N$ and $M$ be two real numbers in $\RR^+\times \overline{\RR^+}$, let $\gamma>0$, and $\hmu_{n}$ be the empirical mean of $n$ random variables i.i.d.\ according to the distribution $\nu_{b'^{-1}(\mu)}$. Then, for $x\leq \mu$,
\begin{align}\label{eq:kl-1}
\begin{split}
    \PP(\exists N\leq n \leq M, \hmu_n\leq x)&\leq e^{-N\cdot \kl(x,\mu)},  \\
    \PP(\exists N\leq n \leq M, \hmu_n\leq x)&\leq e^{-N(x-\mu)^2/(2V)}.
\end{split}
\end{align}
Meanwhile, for every $x\geq \mu$, 
\begin{align}\label{eq:kl-2}
    \PP(\exists N\leq n \leq M, \hmu_n\geq x)&\leq e^{-N(x-\mu)^2/(2V)}.
\end{align}
\end{lemma}
\begin{lemma}[Tail Bound for Gaussian Distribution]
\label{lem:gaussian-tail} 
 For a random variable $Z\sim\cN(\mu,\sigma^2)$, 
\begin{equation}
\label{eq:lowerbound-G}
  \frac{e^{-z^2/2}}{z\cdot \sqrt{2\pi}} \geq \PP(Z>\mu+z\sigma)\geq \frac{1}{\sqrt{2\pi}}\frac{z}{z^2+1} e^{-\frac{z^2}{2}}.
\end{equation} 
Besides, for $0\leq z\leq 1$,
\begin{align*}
    \PP(Z>\mu+z\sigma)\geq \frac{1}{\sqrt{8\pi}} e^{-\frac{z^2}{2}}. 
\end{align*}
\end{lemma}
\begin{proof}
 \eqref{eq:lowerbound-G} is from \citet{abramowitz1964handbook}. For the second statement, we have that for $0\leq z\leq 1$,
\begin{align*}
    \PP(Z>\mu+z\sigma) \geq \PP(Z>\mu+\sigma) \geq \frac{1}{\sqrt{8\pi}} e^{-\frac{z^2}{2}},
\end{align*}
where the last inequality is due to \eqref{eq:lowerbound-G}.
\end{proof}

\bibliographystyle{ims}
\bibliography{reference}

\begin{thebibliography}{34}
\expandafter\ifx\csname natexlab\endcsname\relax\def\natexlab#1{#1}\fi
\expandafter\ifx\csname url\endcsname\relax
  \def\url#1{\texttt{#1}}\fi
\expandafter\ifx\csname urlprefix\endcsname\relax\def\urlprefix{URL }\fi

\bibitem[{Abeille and Lazaric(2017)}]{abeille2017linear}
\textsc{Abeille, M.} and \textsc{Lazaric, A.} (2017).
\newblock Linear thompson sampling revisited.
\newblock In \textit{Artificial Intelligence and Statistics}. PMLR.

\bibitem[{Abramowitz and Stegun(1964)}]{abramowitz1964handbook}
\textsc{Abramowitz, M.} and \textsc{Stegun, I.~A.} (1964).
\newblock \textit{Handbook of mathematical functions with formulas, graphs, and
  mathematical tables}, vol.~55.
\newblock US Government printing office.

\bibitem[{Agrawal and Goyal(2012)}]{agrawal2012analysis}
\textsc{Agrawal, S.} and \textsc{Goyal, N.} (2012).
\newblock Analysis of thompson sampling for the multi-armed bandit problem.
\newblock In \textit{Conference on learning theory}.

\bibitem[{Agrawal and Goyal(2013)}]{agrawal2013further}
\textsc{Agrawal, S.} and \textsc{Goyal, N.} (2013).
\newblock Further optimal regret bounds for thompson sampling.
\newblock In \textit{Artificial intelligence and statistics}.

\bibitem[{Agrawal and Goyal(2017)}]{agrawal2017near}
\textsc{Agrawal, S.} and \textsc{Goyal, N.} (2017).
\newblock Near-optimal regret bounds for thompson sampling.
\newblock \textit{Journal of the ACM (JACM)} \textbf{64} 30.

\bibitem[{Audibert and Bubeck(2009)}]{audibert2009minimax}
\textsc{Audibert, J.-Y.} and \textsc{Bubeck, S.} (2009).
\newblock Minimax policies for adversarial and stochastic bandits.
\newblock In \textit{COLT}.

\bibitem[{Auer et~al.(2002{\natexlab{a}})Auer, Cesa-Bianchi and
  Fischer}]{auer2002finite}
\textsc{Auer, P.}, \textsc{Cesa-Bianchi, N.} and \textsc{Fischer, P.}
  (2002{\natexlab{a}}).
\newblock Finite-time analysis of the multiarmed bandit problem.
\newblock \textit{Machine learning} \textbf{47} 235--256.

\bibitem[{Auer et~al.(2002{\natexlab{b}})Auer, Cesa-Bianchi, Freund and
  Schapire}]{auer2002nonstochastic}
\textsc{Auer, P.}, \textsc{Cesa-Bianchi, N.}, \textsc{Freund, Y.} and
  \textsc{Schapire, R.~E.} (2002{\natexlab{b}}).
\newblock The nonstochastic multiarmed bandit problem.
\newblock \textit{SIAM journal on computing} \textbf{32} 48--77.

\bibitem[{Auer and Ortner(2010)}]{auer2010ucb}
\textsc{Auer, P.} and \textsc{Ortner, R.} (2010).
\newblock Ucb revisited: Improved regret bounds for the stochastic multi-armed
  bandit problem.
\newblock \textit{Periodica Mathematica Hungarica} \textbf{61} 55--65.

\bibitem[{Bian and Jun(2021)}]{bian2021maillard}
\textsc{Bian, J.} and \textsc{Jun, K.-S.} (2021).
\newblock Maillard sampling: Boltzmann exploration done optimally.
\newblock \textit{arXiv preprint arXiv:2111.03290} .

\bibitem[{Bubeck and Liu(2013)}]{bubeck2013prior}
\textsc{Bubeck, S.} and \textsc{Liu, C.-Y.} (2013).
\newblock Prior-free and prior-dependent regret bounds for thompson sampling.
\newblock In \textit{Advances in Neural Information Processing Systems}.

\bibitem[{Chapelle and Li(2011)}]{chapelle2011empirical}
\textsc{Chapelle, O.} and \textsc{Li, L.} (2011).
\newblock An empirical evaluation of thompson sampling.
\newblock In \textit{Advances in neural information processing systems}.

\bibitem[{Degenne and Perchet(2016)}]{degenne2016anytime}
\textsc{Degenne, R.} and \textsc{Perchet, V.} (2016).
\newblock Anytime optimal algorithms in stochastic multi-armed bandits.
\newblock In \textit{International Conference on Machine Learning}. PMLR.

\bibitem[{Garivier and Capp{\'e}(2011)}]{garivier2011kl}
\textsc{Garivier, A.} and \textsc{Capp{\'e}, O.} (2011).
\newblock The kl-ucb algorithm for bounded stochastic bandits and beyond.
\newblock In \textit{Proceedings of the 24th annual conference on learning
  theory}.

\bibitem[{Garivier et~al.(2018)Garivier, Hadiji, Menard and
  Stoltz}]{garivier2018kl}
\textsc{Garivier, A.}, \textsc{Hadiji, H.}, \textsc{Menard, P.} and
  \textsc{Stoltz, G.} (2018).
\newblock Kl-ucb-switch: optimal regret bounds for stochastic bandits from both
  a distribution-dependent and a distribution-free viewpoints.
\newblock \textit{arXiv preprint arXiv:1805.05071} .

\bibitem[{Harremo{\"e}s(2016)}]{harremoes2016bounds}
\textsc{Harremo{\"e}s, P.} (2016).
\newblock Bounds on tail probabilities in exponential families.
\newblock \textit{arXiv preprint arXiv:1601.05179} .

\bibitem[{Je{\v{r}}{\'a}bek(2004)}]{jevrabek2004dual}
\textsc{Je{\v{r}}{\'a}bek, E.} (2004).
\newblock Dual weak pigeonhole principle, boolean complexity, and
  derandomization.
\newblock \textit{Annals of Pure and Applied Logic} \textbf{129} 1--37.

\bibitem[{Jin et~al.(2021{\natexlab{a}})Jin, Tang, Xu, Huang, Xiao and
  Gu}]{jin2021almost}
\textsc{Jin, T.}, \textsc{Tang, J.}, \textsc{Xu, P.}, \textsc{Huang, K.},
  \textsc{Xiao, X.} and \textsc{Gu, Q.} (2021{\natexlab{a}}).
\newblock Almost optimal anytime algorithm for batched multi-armed bandits.
\newblock In \textit{International Conference on Machine Learning}. PMLR.

\bibitem[{Jin et~al.(2021{\natexlab{b}})Jin, Xu, Shi, Xiao and
  Gu}]{jin2021mots}
\textsc{Jin, T.}, \textsc{Xu, P.}, \textsc{Shi, J.}, \textsc{Xiao, X.} and
  \textsc{Gu, Q.} (2021{\natexlab{b}}).
\newblock {MOTS: Minimax Optimal Thompson Sampling}.
\newblock In \textit{International Conference on Machine Learning}. PMLR.

\bibitem[{Jin et~al.(2021{\natexlab{c}})Jin, Xu, Xiao and Gu}]{jin2021double}
\textsc{Jin, T.}, \textsc{Xu, P.}, \textsc{Xiao, X.} and \textsc{Gu, Q.}
  (2021{\natexlab{c}}).
\newblock Double explore-then-commit: Asymptotic optimality and beyond.
\newblock In \textit{Conference on Learning Theory}. PMLR.

\bibitem[{Kaufmann(2016)}]{kaufmann2016bayesian}
\textsc{Kaufmann, E.} (2016).
\newblock On bayesian index policies for sequential resource allocation.
\newblock \textit{arXiv preprint arXiv:1601.01190} .

\bibitem[{Kaufmann et~al.(2012)Kaufmann, Korda and
  Munos}]{kaufmann2012thompson}
\textsc{Kaufmann, E.}, \textsc{Korda, N.} and \textsc{Munos, R.} (2012).
\newblock Thompson sampling: An asymptotically optimal finite-time analysis.
\newblock In \textit{International conference on algorithmic learning theory}.
  Springer.

\bibitem[{Kim et~al.(2021)Kim, Kim and Paik}]{kim2021doubly}
\textsc{Kim, W.}, \textsc{Kim, G.-s.} and \textsc{Paik, M.~C.} (2021).
\newblock Doubly robust thompson sampling with linear payoffs.
\newblock \textit{Advances in Neural Information Processing Systems}
  \textbf{34}.

\bibitem[{Korda et~al.(2013)Korda, Kaufmann and Munos}]{korda2013thompson}
\textsc{Korda, N.}, \textsc{Kaufmann, E.} and \textsc{Munos, R.} (2013).
\newblock Thompson sampling for 1-dimensional exponential family bandits.
\newblock In \textit{Advances in neural information processing systems}.

\bibitem[{Lai and Robbins(1985)}]{lai1985asymptotically}
\textsc{Lai, T.~L.} and \textsc{Robbins, H.} (1985).
\newblock Asymptotically efficient adaptive allocation rules.
\newblock \textit{Advances in applied mathematics} \textbf{6} 4--22.

\bibitem[{Lattimore(2015)}]{lattimore2015optimally}
\textsc{Lattimore, T.} (2015).
\newblock Optimally confident ucb: Improved regret for finite-armed bandits.
\newblock \textit{arXiv preprint arXiv:1507.07880} .

\bibitem[{Lattimore(2016)}]{lattimore2016regret}
\textsc{Lattimore, T.} (2016).
\newblock Regret analysis of the finite-horizon gittins index strategy for
  multi-armed bandits.
\newblock In \textit{Conference on Learning Theory}.

\bibitem[{Lattimore(2018)}]{lattimore2018refining}
\textsc{Lattimore, T.} (2018).
\newblock Refining the confidence level for optimistic bandit strategies.
\newblock \textit{The Journal of Machine Learning Research} \textbf{19}
  765--796.

\bibitem[{Lattimore and Szepesv{\'a}ri(2020)}]{lattimore2018bandit}
\textsc{Lattimore, T.} and \textsc{Szepesv{\'a}ri, C.} (2020).
\newblock \textit{Bandit algorithms}.
\newblock Cambridge University Press.

\bibitem[{Maillard et~al.(2011)Maillard, Munos and Stoltz}]{maillard2011finite}
\textsc{Maillard, O.-A.}, \textsc{Munos, R.} and \textsc{Stoltz, G.} (2011).
\newblock A finite-time analysis of multi-armed bandits problems with
  kullback-leibler divergences.
\newblock In \textit{Proceedings of the 24th annual Conference On Learning
  Theory}.

\bibitem[{M{\'e}nard and Garivier(2017)}]{menard2017minimax}
\textsc{M{\'e}nard, P.} and \textsc{Garivier, A.} (2017).
\newblock A minimax and asymptotically optimal algorithm for stochastic
  bandits.
\newblock In \textit{International Conference on Algorithmic Learning Theory}.

\bibitem[{Russo and Van~Roy(2014)}]{russo2014learning}
\textsc{Russo, D.} and \textsc{Van~Roy, B.} (2014).
\newblock Learning to optimize via posterior sampling.
\newblock \textit{Mathematics of Operations Research} \textbf{39} 1221--1243.

\bibitem[{Wang and Chen(2018)}]{wang2018thompson}
\textsc{Wang, S.} and \textsc{Chen, W.} (2018).
\newblock Thompson sampling for combinatorial semi-bandits.
\newblock In \textit{International Conference on Machine Learning}.

\bibitem[{Zhang(2021)}]{zhang2021feel}
\textsc{Zhang, T.} (2021).
\newblock Feel-good thompson sampling for contextual bandits and reinforcement
  learning.
\newblock \textit{arXiv preprint arXiv:2110.00871} .

\end{thebibliography}

\end{document}